\documentclass{article}
\usepackage{ijcai23}

\usepackage{times}
\usepackage{soul}
\usepackage{url}
\usepackage[hidelinks]{hyperref}
\usepackage[utf8]{inputenc}
\usepackage[small]{caption}
\usepackage{graphicx}
\usepackage{amsmath}
\usepackage{amsthm}
\usepackage{booktabs}
\usepackage{algorithm}
\usepackage{algorithmic}
\usepackage[switch]{lineno}

\usepackage{enumitem}
\usepackage{listings}
\setlist[itemize]{leftmargin=*}
\usepackage{amsmath}
\usepackage{amsthm,amssymb}

\let\argmax\relax
\let\argmin\relax

\usepackage{amsmath}\allowdisplaybreaks
\usepackage{amsfonts,bm}
\usepackage{amssymb}
\usepackage{dsfont}
\usepackage{amsthm}

\theoremstyle{plain}

\newtheorem{theorem}{Theorem}[section]
\newtheorem{proposition}{Proposition}[section]
\newtheorem{corollary}{Corollary}[section]
\newtheorem{lemma}{Lemma}[section]
\newtheorem{definition}{Definition}[section] 

\newtheorem{assumption}{Assumption}

\def\cA{{\mathcal{A}}}

\def\cM{{\mathcal{M}}}

\def\cO{{\mathcal{O}}}

\def\cS{{\mathcal{S}}}

\DeclareMathOperator*{\argmax}{arg\,max}
\DeclareMathOperator*{\argmin}{arg\,min}

\usepackage{booktabs} 
\usepackage{float}
\usepackage{tikz} 
\usepackage{bbm}
\usepackage{wrapfig}
\usepackage{graphicx}
\usepackage{subcaption}
\usepackage{color}
\usepackage{xcolor}
\newcommand{\revise}[1]{\color[RGB]{0,0,0}#1}
\newcommand{\citet}[1]{\citeauthor{#1} (\citeyear{#1})}
\newcommand{\citep}[1]{\cite{#1}}

\title{DPMAC: Differentially Private Communication for Cooperative Multi-Agent Reinforcement Learning}

\author{
Canzhe Zhao$^{1,*}$
\and
Yanjie Ze$^{2,*}$\and
Jing Dong$^3$\and
Baoxiang Wang$^3$\And
Shuai Li$^{1,\dag}$
\affiliations
$^1$John Hopcroft Center for Computer Science, Shanghai Jiao Tong University\\
$^2$Department of Computer Science and Engineering, Shanghai Jiao Tong University\\
$^3$School of Data Science, The Chinese University of Hong Kong, Shenzhen
\emails
\{canzhezhao, zeyanjie, shuaili8\}@sjtu.edu.cn,
jingdong@link.cuhk.edu.cn,
bxiangwang@cuhk.edu.cn
}

\begin{document}

\maketitle

{
\renewcommand{\thefootnote}{\fnsymbol{footnote}}
\footnotetext[1]{The first two authors contributed equally to this work.}
\footnotetext[2]{Corresponding author.}
}

\begin{abstract}
Communication lays the foundation for cooperation in human society and in multi-agent reinforcement learning (MARL). Humans also desire to maintain their privacy when communicating with others, yet such privacy concern has not been considered in existing works in MARL. To this end, we propose the \textit{differentially private multi-agent communication} (DPMAC) algorithm, which protects the sensitive information of individual agents by equipping each agent with a local message sender with rigorous $(\epsilon, \delta)$-differential privacy (DP) guarantee. In contrast to directly perturbing the messages with predefined DP noise as commonly done in privacy-preserving scenarios, we adopt a stochastic message sender for each agent respectively and incorporate the DP requirement into the sender, which automatically adjusts the learned message distribution to alleviate the instability caused by DP noise. Further, we prove the existence of a Nash equilibrium in cooperative MARL with privacy-preserving communication, which suggests that this problem is game-theoretically learnable. Extensive experiments demonstrate a clear advantage of DPMAC over baseline methods in privacy-preserving scenarios.
\end{abstract}

\section{Introduction}
Multi-agent reinforcement learning (MARL) has shown remarkable achievements in many real-world
applications such as sensor networks \citep{zhang2011coordinated}, autonomous driving \citep{shwartz2016safe}, and traffic control \citep{wei2019colight}. 
To mitigate non-stationarity when training the multi-agent system, centralized training and decentralized execution (CTDE) paradigm is proposed. 
The CTDE paradigm yet faces the hardness to enable complex cooperation and coordination for agents during execution due to the inherent partial observability in multi-agent scenarios \citep{NDQ_2020}.
To make agents cooperate more efficiently in complex partial observable environments, communication between agents has been considered. 
Numerous works proposed differentiable communication methods between agents, 
which can be trained in an end-to-end manner, for more efficient cooperation among agents \citep{DBLP:conf/nips/FoersterAFW16,ATOC_2018,TarMAC_2019,I2C_2020,intention_sharing_2021,NDQ_2020}. 

However, the advantages of communication, resulting from full information sharing, come with the possible privacy leakage of individual agents for both broadcasted and one-to-one messages.
{\revise Therefore, in practice, one agent may be unwilling to fully share its private information with other agents even though in \textit{cooperative} scenarios. For instance, if we  train and deploy an MARL-based autonomous driving system, each autonomous vehicle involved in this system could be regarded as an agent and all vehicles work together to improve the safety and efficiency of the system. Hence, this can be regarded as a cooperative MARL scenario \citep{shwartz2016safe,YangNIFZ20}. However, owners of autonomous vehicles may not 
allow their vehicles 
to send private information to other vehicles without any desensitization since this may divulge their private information such as their personal life routines \citep{HassanRC20}. 
}
Hence, a natural question arises:

\textit{Can the MARL algorithm with communication under the CTDE framework be endowed with both the rigorous privacy guarantee and the empirical efficiency?}

To answer this question, we start with a simple motivating example called \textit{single round binary sums},
where several players attempt to guess the bits possessed by others and they can share their own information by communication. 
In Section \ref{section: motivation example}, we show that a local message sender using the randomized response mechanism allows an privacy-aware receiver to correctly calculate the binary sum in a privacy-preserving way. From the example, we gain two insights: 1) The information is not supposed to be aggregated likewise in previous communication methods in MARL \citep{TarMAC_2019,I2C_2020}, as a trusted data curator is not available in general. On the contrary, privacy is supposed to be achieved locally for every agent; 2) Once the agents know a priori, that certain privacy constraint exists, they could adjust their inference on the noised messages. These two insights indicate the principles of our privacy-preserving communication scheme that we desire \textit{a privacy-preserving local sender} and \textit{a privacy-aware receiver}.

Our algorithm, \textit{differentially private multi-agent communication} (DPMAC), instantiates the described principles. More specifically, for the sender part, each agent is equipped with a \textit{local} sender which ensures differential privacy (DP) \citep{DP_2006} by performing an additive Gaussian noise. 
{\revise The message sender in DPMAC is local in the sense that each agent is equipped with its own message sender, which is only used to send its own messages. Equipped with this local sender, DPMAC is able to not only protect the privacy of communications between agents but also satisfy different privacy levels required from different agents.}
In addition, {\revise the sender} adopts the Gaussian distribution to represent the message space and sample stochastic messages from the learned distribution.
However, it is known that the DP noise may impede the original learning process \citep{dwork2014algorithmic,alvim2011differential}, resulting in unstable or even divergent algorithms, 
especially for deep learning-based methods \citep{abadi2016deep,GS-WGAN_2020}. 
To cope with this issue, we incorporate the noise variance into the representation of the message distribution, so {\revise that} the agents could learn to adjust the message distribution automatically according to varying noise {\revise scales}. For the receiver part, due to the gradient chain between the sender and the receiver, our receiver naturally utilizes the privacy-relevant information hidden in the gradients. This implements the privacy-aware receiver described in the motivating example.

When protecting privacy in communication is required in a cooperative game,
the game is \textit{not} purely cooperative anymore since each player involved will face a trade-off between the team utility and its personal privacy. 
To analyze the convergence of cooperative {\revise games} with privacy-preserving communication, 
we first define a single-step game, namely {\revise the} \textit{collaborative game with privacy} (CGP).
We prove that under some mild assumptions of the players' value functions, 
CGP could be transformed into a potential game \citep{monderer1996potential}, subsequently leading to the existence of a Nash equilibrium (NE). 
With this property, NE could also be proven to exist in
the single round binary sums. 
Furthermore, 
we extend the single round binary sums into a multi-step game called \textit{multiple round sums} using the notion of Markov potential game (MPG) \citep{leonardos2021global}. 
{\revise Inspired by} \citet{DBLP:conf/iclr/MacuaZZ18}
and modeling the privacy-preserving communication as part of the agent action, we prove the existence of NE, which indicates that the multi-step game with privacy-preserving communication could be learnable.

To validate the effectiveness of DPMAC, extensive experiments are conducted in multi-agent particle environment (MPE) \citep{MADDPG_2017}, including cooperative navigation, cooperative communication and navigation, and predator-prey tasks. 
Specifically, in privacy-preserving scenarios, DPMAC significantly outperforms baselines. Moreover, even without any privacy {\revise constraints}, DPMAC could also gain competitive performance against baselines.

To sum up, the contributions of this work are threefold:
\begin{itemize}
    \item To the best of our knowledge, we make the first attempt to develop a framework for private communication in MARL, named DPMAC, 
    with provable $(\epsilon,\delta)$-DP guarantee.
    \item  
    We prove the existence of the Nash equilibrium for cooperative games with privacy-preserving communication, showing that these games are game-theoretically learnable.
    \item Extensive experiments 
    show that DPMAC clearly outperforms baselines in privacy-preserving scenarios and gains competitive performance in non-private scenarios.
\end{itemize}

\section{Related Work}
\paragraph{Learning to communicate in MARL}
Learning communication protocols in MARL by backpropagation and end-to-end training has achieved great advances in
recent years \citep{SukhbaatarSF16,DBLP:conf/nips/FoersterAFW16,ATOC_2018,TarMAC_2019,NDQ_2020,I2C_2020,intention_sharing_2021,RangwalaW20,VBC_2019,SinghJS19,zhang2020succinct,zhang2021taming,lin2021learningto,peng2017multiagent}.  
Amongst these works, \citet{SukhbaatarSF16} propose CommNet as the first differentiable communication framework for MARL. 
Further, TarMAC \citep{TarMAC_2019} and ATOC \citep{ATOC_2018} utilize the attention mechanism to extract useful information as messages. 
I2C \citep{I2C_2020} makes the first attempt to enable agents to learn one-to-one communication via causal inference. 
\citet{NDQ_2020} propose NDQ, which learns nearly decomposable value functions to reduce the communication overhead.
\citet{intention_sharing_2021} consider sharing an imagined trajectory as {\revise an} intention for effectiveness. 
{\revise Besides, 
to communicate in the scenarios with limited bandwidth, some works consider learning to send compact and informative messages in MARL via minimizing the entropy of messages between agents using information bottleneck methods \citep{WangHYQ0R20,abs-2207-00088,abs-2112-10374,LiYLL21}.}
While learning effective communication in MARL has been extensively investigated, existing communication algorithms potentially leave the privacy of each agent vulnerable to information attacks.

\paragraph{Privacy preserving in RL}
With wide attention on reinforcement learning (RL) algorithms and applications in recent years, so have concerns about their privacy.
\citet{sakuma2008privacy} consider privacy in the distributed RL problem and utilize cryptographic tools to protect the private state-action-state triples. Algorithmically, \citet{balle2016differentially} make the first attempt to establish a policy evaluation algorithm with DP guarantee, where the Monte-Carlo estimates are perturbed with Gaussian noises. 
\citet{wang2019privacy} generalize the results to Q-learning, where functional noises are added to protect the reward functions. 
Theoretically, \citet{garcelon2021local} study regret minimization of finite-horizon  Markov decision processes (MDPs) with DP guarantee in the tabular case. 
In a large or continuous state space where function approximation is required,
\citet{liao2021locally} and \citet{zhou2022differentiallyrl} subsequently take the first step to establish the sublinear regret in linear mixture MDPs.
Meanwhile, a large number of works focus on preserving privacy in multi-armed bandits \citep{tao2022optimal,tenenbaumKMS2021differentially,dubey21no,zheng2022locally,dubeyP2020differentially,tossouD17achieving}.

Privacy is also studied in recent literature on MARL and multi-agent system. \citet{ye2020differential_advising} study differential advising for value-based agents, which share action values as the advice, largely differing in both the communication framework and the CTDE framework. \citet{dong2020distributed}  propose an average consensus algorithm with {\revise a} DP guarantee in the multi-agent system.

\section{Preliminaries}
We consider a fully cooperative MARL problem where $N$ agents work collaboratively to maximize the joint rewards. The underlying environment can be captured by a decentralized partially observable Markov decision process (Dec-POMDP), denoted by the tuple $\langle\mathcal{S},  \mathcal{A}, \mathcal{O}, \mathcal{P}, \mathcal{R}, \gamma \rangle$.
Specifically, $\cS$ is the global state space, $\cA=\prod^N_{i=1}\cA_i$ is the joint action space, $\cO=\prod^N_{i=1}\cO_i$ is the joint observation space, $\mathcal{P}\left(s^{\prime} \mid s, \boldsymbol{a}\right):=\mathcal{S} \times \mathcal{A} \times \mathcal{S} \rightarrow[0,1]$ determines the state transition dynamics, $\mathcal{R}(s, \boldsymbol{a}): \mathcal{S} \times \mathcal{A} \rightarrow \mathbb{R}$
is the reward function, and $\gamma\in[0,1)$ is the discount factor.
Given a joint policy $\boldsymbol{\pi}=\{\pi_{i}\}^N_{i=1}$, the joint action-value function at time $t$ is  $Q^{\boldsymbol{\pi}}\left(s^{t}, \boldsymbol{a}^{t}\right)=\mathbb{E}\left[G^{t} \mid s^{t},\boldsymbol{a}^{t},\boldsymbol{\pi} \right]$, where $G^{t}=\sum_{i=0}^{\infty} \gamma^{i} \mathcal{R}^{t+i}$ is the cumulative reward, and $\boldsymbol{a}^t=\{a_i^t\}^N_{i=1}$ is the joint action. 
The ultimate goal of the agents is to find an optimal policy $\boldsymbol{\pi}^\ast$ which maximizes $Q^{\boldsymbol{\pi}}\left(s^{t},\boldsymbol{a}^{t}\right)$.

Under the aforementioned cooperative setting, we study the case where agents are allowed to communicate with a joint message space $\cM=\prod^N_{i=1}\cM_i$. 
When the communication is unrestricted, the problem is reduced to a single-agent RL problem, which effectively solves the challenge posed by partially observable states, but puts the individual agent's privacy at risk. 
To overcome the challenges of privacy and partial observable states simultaneously, we investigate algorithms that maximize the cumulative rewards while satisfying DP, given in the following definition.

\begin{definition}[$(\epsilon,\delta)$-DP, \cite{DP_2006}]
 \label{def: epsilon, delta-DP}
 A randomized mechanism $f: \mathcal{D} \to$ $\mathcal{Y}$ satisfies $(\epsilon, \delta)$-differential privacy if for any neighbouring datasets $D, D^{\prime} \in$ $\mathcal{D}$ and $S \subset \mathcal{Y}$, it holds that
$
\operatorname{Pr}[f(D) \in S] \leq e^{\epsilon} \operatorname{Pr}\left[f\left(D^{\prime}\right) \in S\right]+\delta
$.
\end{definition}
DP offers a mathematically rigorous way to quantify the privacy of an algorithm \citep{DP_2006}. An algorithm is said to be ``privatized'' under the notion of DP if it is statistically hard to infer the presence of an individual data point in the dataset by observing the output of the algorithm. {\revise More intuitively, an algorithm satisfies DP if it provides nearly the same outputs given the neighbouring input datasets (\textit{i.e.}, $\operatorname{Pr}[f(D) \in S]\approx\operatorname{Pr}\left[f\left(D^{\prime}\right) \in S\right]$), which hence protects the sensitive information from the curious attacker.}

With DP, each agent $i$ is assigned with a privacy budget $\epsilon_i$, which is negatively correlated to the level of privacy protection. Then we have $\boldsymbol{\epsilon}=\{\epsilon_i\}^N_{i=1}$ as the set of all privacy budgets. In addition to maximizing the joint rewards as usually required in cooperative MARL, the messages sent from agent $i$ are also required to satisfy the privacy budget $\epsilon_i$ with {\revise probability at least $1-\delta$}.

\section{Motivating Example}
\label{section: motivation example}
Before introducing our communication framework, we first investigate a motivating example, which {\revise is a \textit{cooperative} game and} inspires the design principles of private communication mechanisms in MARL. The motivating example is a simple yet interesting game, called \textit{single round binary sums}. 
The game is extended from the example provided in \cite{cheu2021differential} for analyzing the shuffle model, while we illustrate the game from the perspective of multi-agent systems. 
We note that though this game is one-step, which is different from the sequential decision process like MDP, it is {\revise illustrative} enough to show how the communication protocol works as a tool to achieve a better trade-off between privacy and utility.

Assume that there are $N$ agents involved in this game.
Each agent $i \in [N]$ has a bit $b_i\in\{0,1\}$ and can tell other agents the information about its bit by communication. The objective {\revise of} the game is for every agent to guess $\sum_i b_i$, the sum of the bits of all agents. Namely, each agent $i$ makes a guess $g_i$ and the utility of the agent is to maximize $r_i=-|\sum_{j}b_j - \mathbb{E}[g_i]|$. The (global) reward of this game is the sum of the utility over all agents, \textit{i.e.}, $\sum_i r_i$.

Without loss of generality, we write the guess $g_i$ into $g_i = \sum_{j\neq i}y_{ij}+ b_i$, where $y_{ij}$ is the guessed bit of agent $j$ by agent $i$.
If all agents share their bits without covering up, the guessed bit $y_{ij}$ will obviously be equal to $b_j$ and all agents attain an optimal return. {\revise Hence this game is fully cooperative under no privacy constraints}. 
{\revise However, the optimal strategy is  under the assumption that \textit{everyone is altruistic to share their own bits}. }

To preserve the privacy in communication, the message (\textit{i.e.}, the sent bit) could be randomized using \textit{randomized response}, which perturbs the bit $b_i$ with probability $p$, as shown below:
\begin{align*}
    x_i=\mathcal{R}_{\mathrm{RR}} \left(b_{i}\right):= \begin{cases}\operatorname{Ber}(1 / 2) & \text { with probability } p \\ b_{i} & \text { otherwise }\,,\end{cases}
\end{align*}
where $x_i$ is the random message and $\operatorname{Ber}(\cdot)$ indicates the Bernoulli distribution. Under our context, $\mathcal{R}_{\mathrm{RR}}$ is a \textit{privacy-preserving message sender}, 
whose privacy guarantee is guaranteed by the following proposition.

\begin{proposition}[\cite{DBLP:conf/crypto/BeimelNO08}]
\label{theorem: randomized response privacy}
Setting $p= \frac{2}{e^\epsilon + 1}$ in $\mathcal{R}_{RR}$ suffices for $(\epsilon,0)$-differential privacy. 
\end{proposition}

When each agent is equipped with such a privacy-preserving sender $\mathcal{R}_{RR}$ while adhering to the originally optimal strategy (\textit{i.e.}, believing what others tell and doing the guess), all agents would make an inaccurate guess. The bias of the guess denoted as $\text{err}_i$ caused by $\mathcal{R}_{RR}$ is then
\begin{align*}
\text{err}_i&=\mathbb{E}[g_i] - \sum_i b_i=   \sum_{j\neq i}\mathbb{E}[x_{j}-b_j] =p\sum_{j\neq i}(\frac{1}{2}-b_j)\\
&=\frac{p(N-1)}{2}-p\sum_{j\neq i}b_j\,.
\end{align*}
Without any prior knowledge, the bias could not be reduced for $(\epsilon,0)$-DP algorithms. 
However, if the probability $p$ of perturbation is set as prior common knowledge for all agents before the game starts, things will be different. One could transform the biased guess into
\begin{align*}
g_i^\mathcal{A}=\mathcal{A}_{\mathrm{RR}}(\vec{x}_{-i}) &:=\frac{1}{1-p}\left(\sum_{j\neq i} x_{j}-(N-1) p / 2\right)\,,
\end{align*}
where $\vec{x}_{-i}=[x_1,\ldots,x_{i-1},x_{i+1},\ldots, x_N]^\top$ denote the messages received by agent $i$. Then the estimate will be unbiased as
\begin{align*}
\mathbb{E}\left[g_i^\mathcal{A}\right] &=\frac{1}{1-p}\left(\mathbb{E}\left[\sum_{j\neq i} x_{j} \right]-\frac{p(N-1) }{2}\right) + b_i = \sum_{i}b_i\,.
\end{align*}

This example inspires that a communication algorithm could be both privacy-preserving and efficient. 
From the perspective of privacy, by the post-processing lemma of DP, any post-processing does not affect the original privacy level.
From the perspective of utility, we could eliminate the bias $\text{err}_i$ if the agent is equipped with the receiver $\mathcal{A}_{RR}$ and the prior knowledge $p$ is given.

In general, our motivating example gives two principles for designing privacy-preserving communication frameworks.
First, to prevent sensitive information from being inferred by other curious agents,
we equip each agent with a local message sender with certain privacy constraints. 
Second, given prior knowledge about the privacy requirement of other agents, the receiver could strategically analyze the received noisy messages to statistically reduce errors due to the noisy communication. These two design principles correspond to two parts of our DPMAC framework respectively, \textit{i.e.}, a \textit{privacy-preserving local sender}, and a \textit{privacy-aware receiver}.

\section{Methodology}
Based on our design principles, we now introduce our DPMAC framework, as shown in Figure  \ref{fig:model}. Our framework is general and flexible, which makes it compatible {\revise with} any CTDE method.

\subsection{Privacy-preserving Local Sender with Stochastic Gaussian Messages}\label{sec:method_sender}

In this section, we present the sender's perspective on the privacy guarantee.
At time $t$, for agent $i$, a message function $f_i^s$ is used to generate a message for communication. $f_i^s$ takes a subset of transitions in local trajectory $\tau^t_{i}$ as input, where the subset is sampled uniformly without replacement from $\tau^t_{i}$ (denote the sampling rate as $\gamma_1$). 
This message is perturbed by the Gaussian mechanism with variance $\sigma_i^2$ \citep{DP_2006}.
Agent $i$ then samples a subset of other agents to share this message (denote the sampling rate as $\gamma_2$). The following theorem guarantees the DP of the sender.

\begin{theorem}[Privacy guarantee for DPMAC]
\label{theorem: privacy for dpmac}
Let $\gamma_1,\gamma_2\in (0,1)$, and $C$ be the {\revise $\ell_2$ norm of} the message functions.
For any $\delta>0$ and privacy budget $\epsilon_i$, the communication of agent $i$ satisfies $(\epsilon_i,\delta)$-DP when 
$\sigma^2_i=\frac{14\gamma_2\gamma_1^2NC^2\alpha}{\beta\epsilon_i}$,
if we have 
$\alpha = \frac{\log \delta^{-1}}{\epsilon_i(1-\beta)}+1 \leq 2 \sigma^{\prime 2} \log \left(1 / \gamma_1 \alpha\left(1+\sigma^{\prime 2}\right)\right) / 3+1$ with $\beta\in (0,1)$ and $\sigma^{\prime 2}=\sigma^2_i/(4C^2)\geq 0.7$
.
\end{theorem}

\begin{figure*}[t]
    \centering
    \vspace{-0.3in}
    \includegraphics[width=0.9\textwidth]{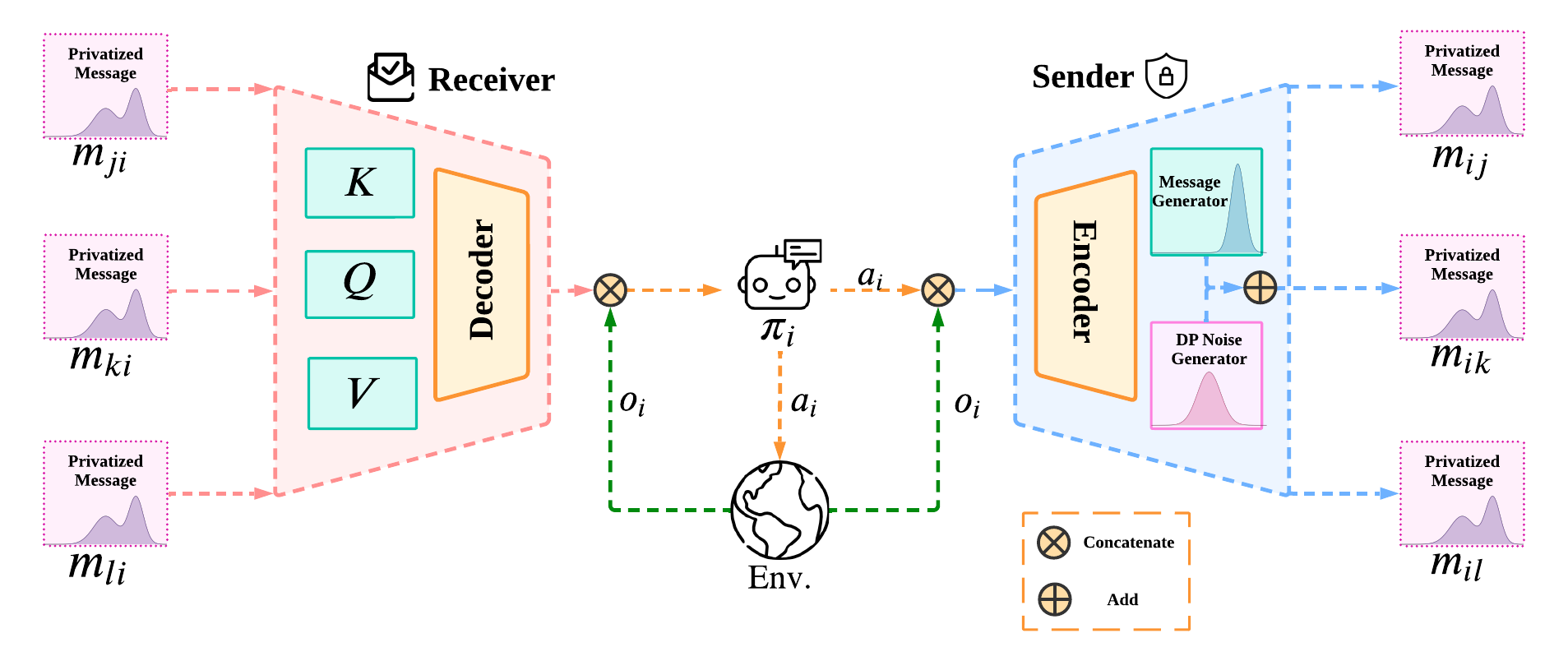}
    \caption{The overall structure of DPMAC.  The message receiver of agent $i$ integrates other agents' messages $\{m_{ji},m_{ki},m_{li}\}$ with the self-attention mechanism and the integrated message is fed into the policy $\pi_i$ together with the observation $o_i$.  Agent $i$ interacts with the environment by  taking action $a_i$. Then $o_i$ and $a_i$ are concatenated and encoded by a privacy-preserving message sender and sent to other agents. }
    \label{fig:model}
\end{figure*}

With Theorem \ref{theorem: privacy for dpmac}, one can directly translate a non-private MARL with a communication algorithm into a private one. However, as we shall see in our experiment section, directly injecting the privacy noise into existing MARL with communication algorithms may lead to serious performance degradation. 
In fact, the injected noise might jeopardize the useful information incorporated in the messages, or even leads to meaningless messages. 
To alleviate the negative impacts of the injected privacy noise on the cooperation between agents, we adopt a stochastic message sender in the sense that the messages sent by our sender are sampled from a learned message distribution.
This makes DPMAC different from existing works in MARL that communicate through deterministic messages \citep{SukhbaatarSF16,DBLP:conf/nips/FoersterAFW16,ATOC_2018,TarMAC_2019,I2C_2020,intention_sharing_2021}.

In the following, we drop the dependency 
of parameters 
on $t$ when it is clear from the context.
Without loss of generality, let the message distribution be multivariate Gaussian and let $p_i$ be the message sampled from the message distribution $\mathcal{N}(\mu_i, \Sigma_i)$, where $\mu_i = f_i^{\mu}(o_i,a_i;{\theta}_i^\mu)$ and  $\Sigma_i = f_i^{\sigma}(o_i,a_i;{\theta}_i^\sigma)$ are the mean vector and covariance matrix learned by the sender, and ${\theta}_i^\mu$ and ${\theta}_i^\sigma$ are the parameters of the sender's neural networks. Then ${\theta}_i^{\mu}$ and ${\theta}_i^{\sigma}$ will be optimized towards making all the agents to send more effective messages to encourage better team cooperation and gain higher team rewards.
For notational convenience, let ${\theta}_i^s=[{\theta}_i^{\mu\top},{\theta}_i^{\sigma\top}]^\top$.
Then the sent privatized message $m_{i}=p_i+u_i$ where $u_i\sim \mathcal{N}(0, \sigma^2_i\mathbf{I}_d)$ is the additive privacy noise. It is clear that $m_{i}\sim\mathcal{N}(\mu_i, \Sigma_i+\sigma^2_i\mathbf{I}_d)$ since $p_i$ is independent from $u_i$.
Counterfactually, 
let $m^\prime_{i}\sim\mathcal{N}(\mu^\prime_i, \Sigma^\prime_i)$ be the sent message when it was not under any privacy constraint, where $\mu^\prime_i = f_i^{\mu}(o_i,a_i;{\theta}_i^{\mu\prime})$ and  $\Sigma^\prime_i = f_i^{\sigma}(o_i,a_i;{\theta}_i^{\sigma\prime})$.

Let the optimal message distribution be $\mathcal{N}(\mu^\ast_i,\Sigma^\ast_i)$. We are interested to characterize ${\theta}_i^{s^\prime}$ and ${\theta}_i^{s}$. By the optimality of $\mu^\ast_i,\Sigma^\ast_i$,
\begin{align}\label{eq:eq_wo_dp}
    {\theta}_i^{s^\prime}&=\argmin_{\theta}{D}_{\operatorname{KL}}(\mathcal{N}(\mu^\prime_i, \Sigma^\prime_i) \| \mathcal{N}(\mu^\ast_i,\Sigma^\ast_i))\notag\\
    &=\argmin_{\theta}\log \frac{|\Sigma^\ast_i|}{|\Sigma^\prime_i|} +\operatorname{tr}\{\Sigma^{\ast-1}_i\Sigma^\prime_i\}+\|\mu^\prime_i-\mu^\ast_i\|^2_{\Sigma^{\ast-1}_i}\,.
\end{align}
Then under the privacy constraints, the stochastic sender will learn ${\theta}_i^s$ such that 
\begin{align}\label{eq:eq_w_dp}
    {\theta}_i^s=\argmin_{\theta}&{D}_{\operatorname{KL}}(\mathcal{N}(\mu_i, \Sigma_i+\sigma^2_i\mathbf{I}_d) \| \mathcal{N}(\mu^\ast_i,\Sigma^\ast_i))\notag\\
    =\argmin_{\theta}&\log \frac{|\Sigma^\ast_i|}{|\Sigma_i+\sigma^2_i\mathbf{I}_d|} +\operatorname{tr}\{\Sigma^{\ast-1}_i(\Sigma_i+\sigma^2_i\mathbf{I}_d)\}\notag\\
    &+\|\mu_i-\mu^\ast_i\|^2_{\Sigma^{\ast-1}_i}\,.
\end{align}
Through {\revise Equation} (\ref{eq:eq_w_dp}),
it is possible to directly incorporate the distribution of privacy noise into the optimization process of the sender to help to learn ${\theta}_i^s$ such that ${D}_{\operatorname{KL}}(\mathcal{N}(\mu_i, \Sigma_i+\sigma^2_i\mathbf{I}_d) \| \mathcal{N}(\mu^\ast_i,\Sigma^\ast_i))\leq {D}_{\operatorname{KL}}(\mathcal{N}(\mu^\prime_i, \Sigma^\prime_i) \| \mathcal{N}(\mu^\ast_i,\Sigma^\ast_i))$, which means that the sender could learn to send private message $m_{i}=p_i+u_i$ that is at least as effective as the non-private message
$m^\prime_{i}$. In this manner, the performance degradation is expected to be well alleviated.


\subsection{Privacy-aware Message Receiver}

As shown in our motivating example, the message receiver with knowledge a priori could statistically reduce the communication error in privacy-preserving scenarios. In the practical design, this motivation could be naturally instantiated with the gradient flow between the message sender and the message receiver.

Specifically, agent $i$ first concatenates all the received privatized messages as $\bm{m}_{(-i)i}:=\{m_{ji}\}^N_{j=1, j\neq i}$ and then decodes $\bm{m}_{(-i)i}$ into an aggregated message $q_i=f_i^r(\bm{m}_{(-i)i}\mid \theta_i^r)$ with the decoding function $f_i^r$ parameterized by $\theta_i^r$. Then a similar argument to the policy gradient theorem \citep{sutton99policy} states that
the gradient of the receiver is
\begin{align*}
    \nabla_{\theta_i^r} \mathcal{J}(\theta_i^r)&= \mathbb{E}_{\boldsymbol{\tau}, \boldsymbol{o}, \boldsymbol{a}}\left[\mathbb{E}_{\pi_i}[ \nabla_{\theta_i^r} f_{i}^{r}\left(q_{i} \mid \bm{m}_{(-i)i}\right) \right.\\
    &\left.\cdot\nabla_{q_{i}}\log \pi_{i}\left(a_{i} \mid o_{i},q_{i}\right) Q^{\boldsymbol{\pi}}(\boldsymbol{a}, \boldsymbol{o})]\right]\,,
\end{align*}
where $\mathcal{J}(\theta_i^r)=\mathbb{E}[G^1\mid \bm{\pi}]$ is the cumulative discounted reward from the starting state.
In this way, the receiver could utilize the prior knowledge $\sigma_i$ of the privacy-preserving sender encoded in the gradient during the optimization process. Please refer to Appendix \ref{sec:algo_pseudo_code} and 
Appendix \ref{sec:app:training} for the complete pseudo code of DPMAC, and detailed optimization process of the message senders and receivers, respectively.


\begin{figure*}[!htbp]
    \centering%
    \begin{subfigure}{0.9\textwidth}
     
    \includegraphics[width=1\textwidth]{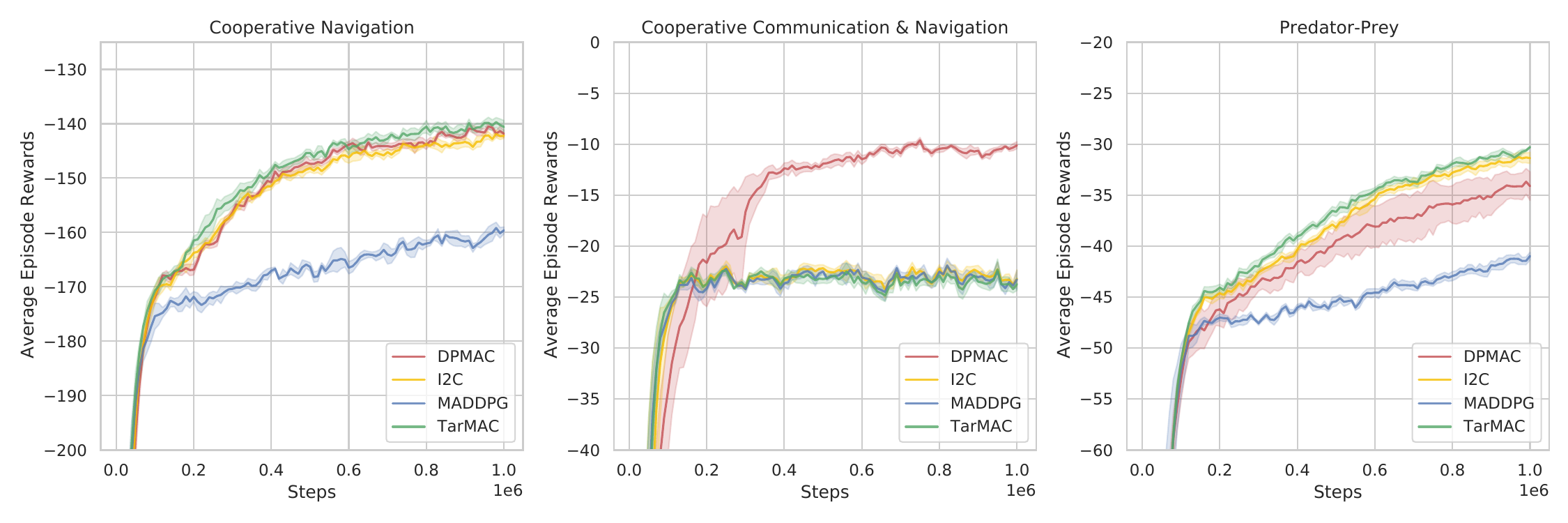}
    \caption{Performance of DPMAC, TarMAC, I2C, and MADDPG on three MPE tasks. 
    }
    \label{fig: communication exp}
    \end{subfigure}
    
    \begin{subfigure}{0.9\textwidth}
    \includegraphics[width=1\textwidth]{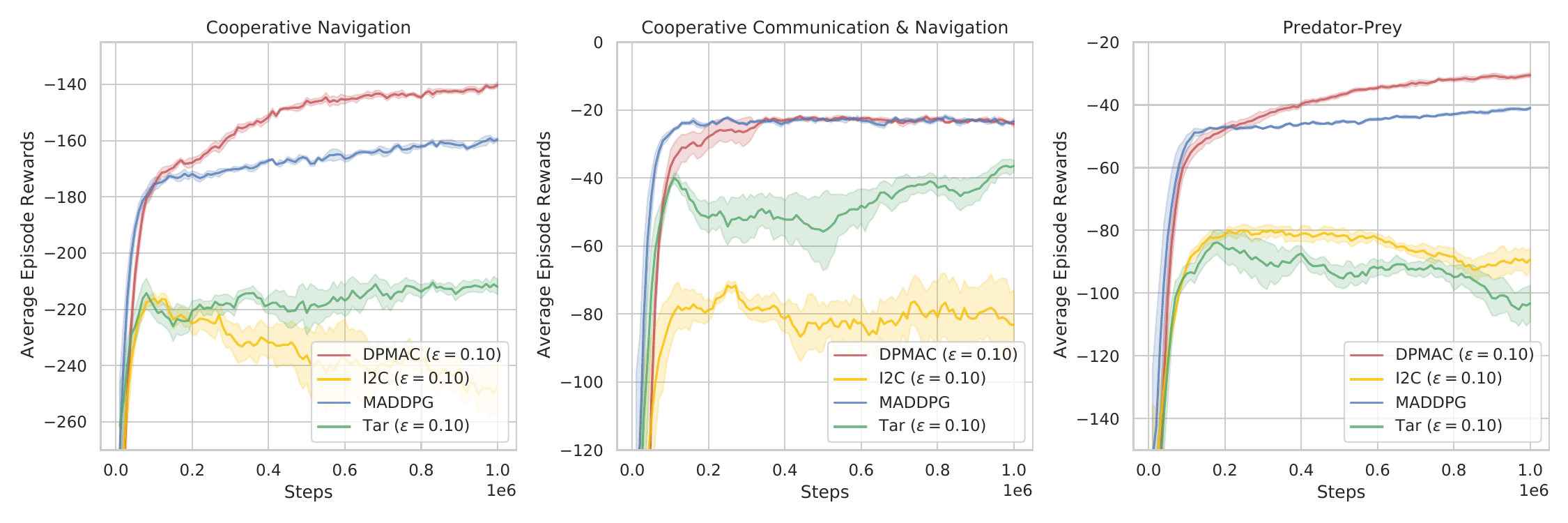}
    \caption{Performance of different algorithms under the privacy budget $\epsilon=0.10$. MADDPG (non-communication) is also displayed for comparison.}
    \label{fig: privact exp 0.1}
    \end{subfigure}
    
     \centering
    \begin{subfigure}{0.9\textwidth}
    \includegraphics[width=1\textwidth]{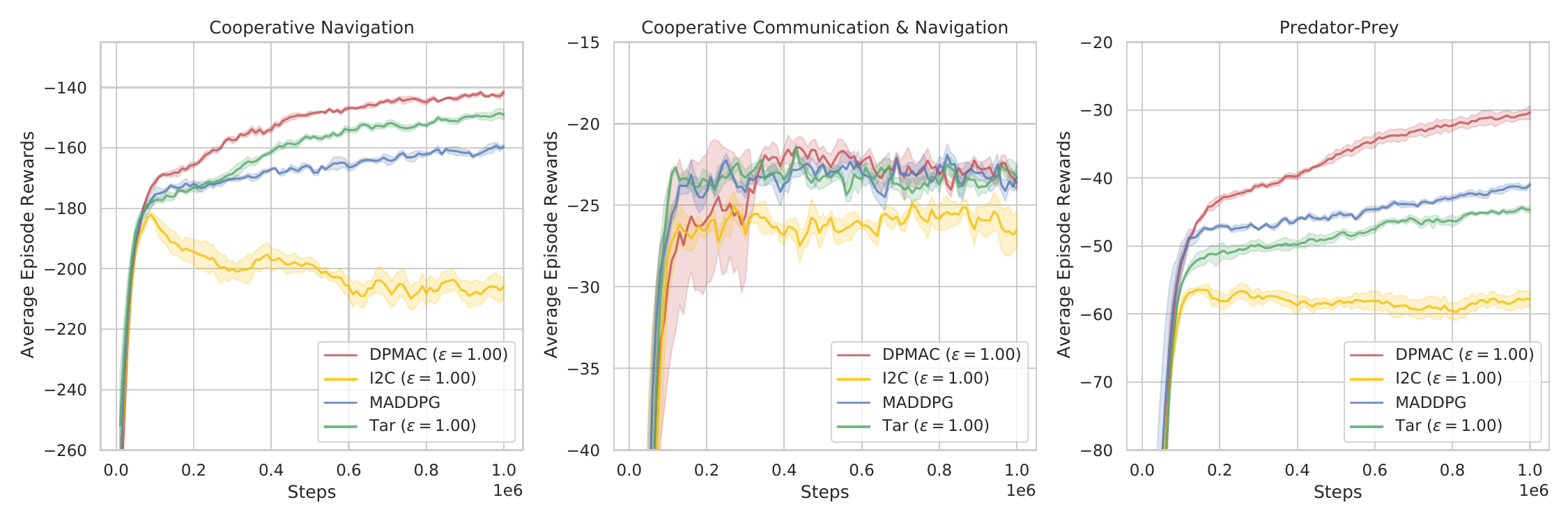}
    \caption{Performance of different algorithms under the privacy budget $\epsilon=1.0$. MADDPG (non-communication) is also displayed for comparison.}
    \label{fig: privacy exp 1.0}
    \end{subfigure}
    
     \centering
    \begin{subfigure}{0.9\textwidth}
    \includegraphics[width=1\textwidth]{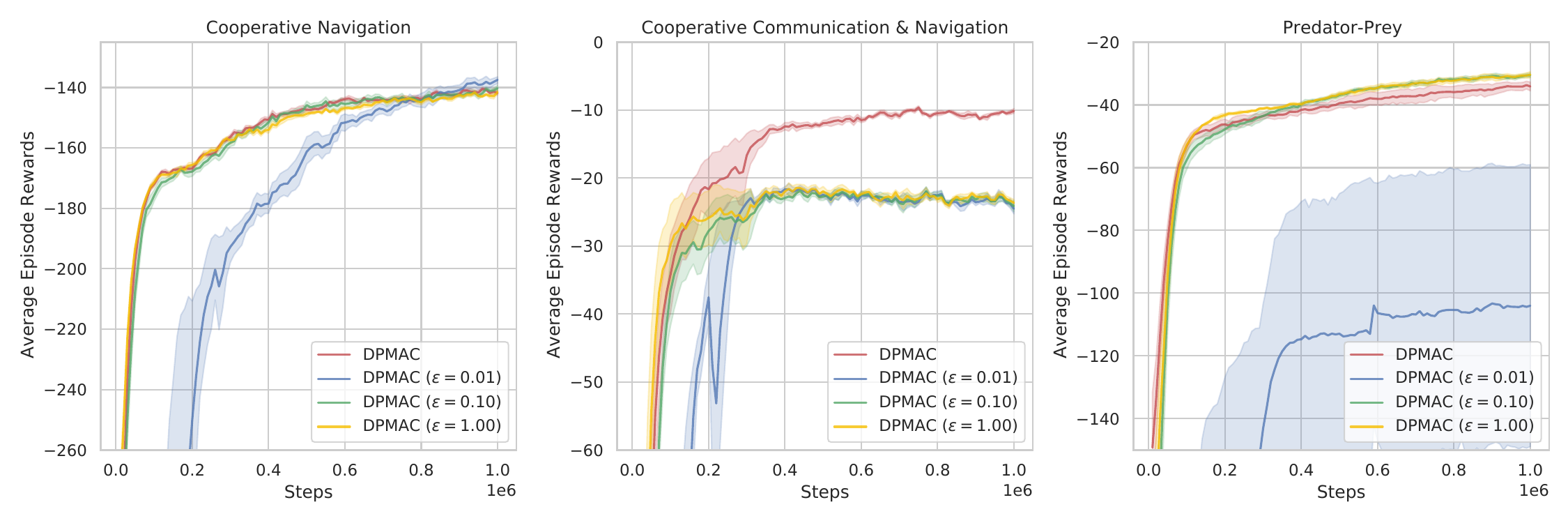}
    \caption{Performance of DPMAC under different privacy budgets ($\epsilon=0.01,0.10,1.00$).}
    \label{fig: privacy exp dpmac}
    \end{subfigure}
    
    \caption{Performance of DPMAC and baseline algorithms. The curves are averaged over $5$ seeds. Shaded areas denote $1$ standard deviation.}\label{fig:fig_single_step_exp}
\end{figure*}

\label{sec: privacy analysis}
\section{Privacy-preserving Equilibrium Analysis}\label{sec:equilibrium_analysis}
As aforementioned, when considering the privacy constraints, the ``cooperative" multi-agent games will \textit{not} be purely cooperative anymore, due to the appearance of the trade-off between the team utility and each player's personal privacy. As the convergence of MARL algorithms could depend on the existence of NE, we first investigate such existence in privacy-preserving single-step games and then extend the result to privacy-preserving multi-step games.

\subsection{Single-step Games}

We study a class of two-player collaborative games, denoted as \textit{collaborative game with privacy (CGP)}. The game 
{\revise involves}
two agents, each equipped with a privacy parameter $p_n$, $n \in \{1, 2\}$. The value of $p_n$ represents the importance of privacy to agent $n$, with the larger value referring to greater importance.
{\revise Let $\cM$ be some message mechanism.}
We {\revise denote} the privacy loss {\revise by} $c^{\cM}(p_n)$, which measures the quantity of the potential privacy leakage and is formally defined in Definition \ref{def: privacy loss function}.
Besides, {\revise let} $b\left( V_n, V^{\cM}_n(p_1,p_2)   \right)$ {\revise be} the utility gained by measuring the gap between private value function $V_n^{\cM}(p_1,p_2)$ 
{\revise and}
non-private value function $V_n$. Then the trade-off between the utility and the privacy is depicted by the total utility function $u_n(p_1,p_2)$ in Equation (\ref{eq: utility function in CGP}). {\revise The} formal definition of CGP is given in Definition \ref{def: CGP}.
See more details in Appendix \ref{appendix: single step game in appendix}.

\begin{definition}[Collaborative game with privacy (CGP)] \label{def: CGP}
The collaborative game with privacy is denoted by a tuple $\langle\mathcal{N}, \Sigma, \mathcal{U}\rangle$, where $\mathcal{N}=\{1,2\}$ is the the set of players, $\Sigma=\left\{p_{1}, p_{2}\right\}$ is the action set with $p_{1}, p_{2} \in[0,1]$ representing the privacy level, and $\mathcal{U}=\left\{u_{1}, u_{2}\right\}$ is the set of utility functions satisfying $\forall n \in \mathcal{N}$,
\begin{align}
\label{eq: utility function in CGP}
 u_{n}\left(p_{1}, p_{2}\right)=B_{n} b\left( V_n, V^{\cM}_n(p_1,p_2)   \right)-C_{n}^{\cM} c^{\cM}\left(p_{n}\right)\,.
\end{align}
\end{definition}
Then the following theorem shows that if changes in the value function of each player can be expressed as a change in their own privacy parameter, then CGP is a potential game and a pure NE thereafter exists. The proof is deferred to Appendix \ref{appendix: single step game in appendix}.
\begin{theorem}[CGP's NE guarantee]
\label{theorem: NE in CGP}
The collaborative game with privacy has at least one non-trivial pure-strategy Nash equilibrium if $\partial_{p_{1}}^{i} V_{1}=\partial_{p_{2}}^{i} V_{2}$, $\forall i \in \{1,2\}$.
\end{theorem}

\paragraph{Equilibrium in single round binary sums}
Let us revisit our motivating example. Armed with the CGP framework, it is immediate that the single round binary sums game guarantees the existence of a NE. {\revise This result is formally stated in Theorem \ref{theorem: single round binary sums} in Appendix \ref{appendix: single step game in appendix}.}

\subsection{Multi-step Games}
We now consider an extended version of single round binary sums named \textit{multiple round sums}. Consider an $N$-player game where player $i$ owns a saving $x_{i,t}$. Rather than sending a binary bit, the agent can choose to give out $b_{i,t}$ at round $t$. Meanwhile, each player $i$ selects privacy level $p_{i,t}$ and sends messages to each other with a sender $f_i^s$ encoding the information of $b_{i,t}$ with the privacy level $p_{i,t}$. The reward of the agent is designed to find a good trade-off between privacy and utility.

We first transform this game into a Markov potential game (MPG), with the reward of each agent transformed into a combination of the team reward and the individual reward. Then with existing theoretical results from \citet{DBLP:conf/iclr/MacuaZZ18}, we present the following result while deferring its proof to Appendix~\ref{appendix: multi-step game}.
   
\begin{theorem}[NE guarantee in multiple round sums]
\label{theorem: NE Guarantee in Multiple Round Sums}
If Assumptions \ref{asmp: 1}, \ref{asmp: 2}, \ref{asmp: 3}, \ref{asmp: 4} (see Appendix \ref{appendix: multi-step game}) are satisfied, our MPG has a NE with potential function $J$ defined as,
    \begin{align}
    J(x_{t}, \pi(x_t))=\sum_{j\in [N]} \left( (1-p_{j,t})b_{j,t}+\alpha x_{j,t}+\beta p_{i,t}  \right)\,.
    \end{align}
\end{theorem}

\section{Experiments}
{\revise In this section, we present the experiment results and corresponding experiment analyses. Please see Appendix \ref{sec:app_addition_exp_analyses} for more detailed  analyses of experiment results.}

\paragraph{Baselines} We implement our DPMAC upon MADDPG \citep{MADDPG_2017} (see Appendix \ref{sec:app:implementation} for concrete implementation details) and evaluate it against TarMAC \citep{TarMAC_2019}, I2C \citep{I2C_2020}, and MADDPG. All algorithms are tested with and without the privacy requirement except for MADDPG, which involves no communication among agents. Since TarMAC and I2C do not have a local sender and have no DP guarantee, we add Gaussian noise to their receiver according to the noise variance specified in Theorem \ref{theorem: privacy for dpmac} for a fair comparison. 
Please see Appendix \ref{sec:app:training} for more training details.

\paragraph{Environments}  We evaluate the algorithms on the multi-agent particle environment (MPE) \citep{mordatch2017emergence}, which is with continuous observation and discrete action space. This environment is commonly used among existing MARL literature \citep{MADDPG_2017,ATOC_2018,I2C_2020,intention_sharing_2021}. We evaluate a wide range of tasks in MPE, including cooperative navigation (CN), cooperative communication and navigation (CCN), and predator prey (PP). 
More details on the environmental settings are given in Appendix \ref{appendix: environment details}.

\paragraph{Experiment results without privacy} 
We first compare DPMAC 
with TarMAC, I2C, and MADDPG on three MPE tasks without the privacy requirement. As shown in Figure  \ref{fig: communication exp}, DPMAC outperforms baselines on CCN task, and has comparable performance on CN and PP tasks. 
More detailed analyses of the experiment results without privacy are deferred to Appendix \ref{sec:further_analyses_wo_privacy}.

\paragraph{Experiment results with privacy} 
We now investigate the performance of algorithms with communication under privacy constraints.
In particular, Figure \ref{fig: privact exp 0.1} and \ref{fig: privacy exp 1.0} show the performance under the privacy budget $\epsilon=0.10,1.0$ and both with $\delta=10^{-4}$. We also include MADDPG as a non-communication baseline method. 
Overall, the privacy constraints impose obvious disturbances to the performance of all algorithms. Specifically, the performance of TarMAC and I2C degenerates significantly and becomes even inferior to the performance of MADDPG. However, in most cases, the performance of DPMAC with privacy constraints only suffers a slight decline and is still superior or comparable to the performance of MADDPG.
See Appendix \ref{sec:further_analyses_w_privacy} for more concrete analyses on the experiment results with privacy.

\paragraph{DPMAC under Different Privacy Budgets} 
In Figure \ref{fig: privacy exp dpmac}, we further present the comparison between the performance of DPMAC under different privacy budgets. When $\epsilon=0.01$, DPMAC still gains remarkable performance on CN and CCN tasks,
while other baselines' performance suffers serious degeneration, as we have analyzed above.
Besides,
on the PP task under the privacy constraint with $\epsilon=0.01$, DPMAC also suffers clear performance degradation. 
Overall, the experiments of DPMAC under different privacy budgets also show that DPMAC could automatically adjust the variance of the stochastic message sender so that it learns a noise-robust message representation. As shown in Figure \ref{fig: privacy exp dpmac}, DPMAC gains very close performance when $\epsilon=0.1$ and $\epsilon=1.0$, 
though the privacy requirements of $\epsilon=0.1$ and $\epsilon=1.0$ differ by one order of magnitude. 
However, one can see large performance gaps for the same baseline algorithms under different $\epsilon$ from Figure \ref{fig: privact exp 0.1} and \ref{fig: privacy exp 1.0}. 
For clarity, we also present these performance gaps of TarMAC and I2C algorithms in Figure \ref{fig: dp tarmac} and \ref{fig: dp i2c}.
Please see Appendix \ref{sec:further_analyses_on_different_privacy} for more detailed analyses of the performance of DPMAC under different privacy budgets.

\section{Conclusion}
\label{sec: conclusion}

In this paper, we study the privacy-preserving communication in MARL. 
Motivated by a simple yet effective example of {\revise the} binary sums game, 
we propose DPMAC, a new efficient communicating MARL algorithm that preserves agents' privacy through DP.
Our algorithm is justified both theoretically and empirically. Besides, to show that the privacy-preserving communication problem is learnable, we analyze the single-step game and the multi-step game via the notion of MPG and show the existence of the Nash equilibrium. This existence further implies the learnability of several instances of MPG under privacy constraints.
Extensive experiments conducted on 3 MPE tasks with varying privacy constraints demonstrate the effectiveness of DPMAC against the baseline methods. 





\bibliographystyle{named}
\bibliography{ijcai23.bib}

\clearpage
\appendix
\section{Algorithm}\label{sec:algo_pseudo_code}
For completeness, we present the pseudo code of DPMAC in Algorithm \ref{alg:algorithm}.
\begin{algorithm*}[tb]
    \caption{Differentially Private Multi-agent Communication (DPMAC)}
    \label{alg:algorithm}
    \textbf{Input}: Privacy parameters of each agent: $\{\epsilon_i\}_{i=1}^{N}$, $\{\delta_i\}_{i=1}^{N}$.
    \begin{algorithmic}[1] 
        \STATE Compute the variance $\sigma_i$ of privacy noises according to Theorem \ref{theorem: privacy for dpmac} for each agent $i=1,\ldots,N$;
        \FOR{\text{episode} $e=1,\ldots,$ \text{max-number-of-episodes}}
            \STATE Initialize the message to be sent as $m_{0,i}=\mathbf{0}$ for each agent $i=1,\ldots,N$;
            \FOR{\text{step} $t=1,\ldots,$ \text{max-number-of-steps}}
                \FOR{\text{each} agent $i=1,\ldots,N$}
                    \STATE \text{Receives observation} $o_{t,i}$;
                    \STATE Receives the messages of all other agents $\boldsymbol{m}_{t-1,(-i) i}:=\left\{m_{t-1,j}\right\}_{j=1, j \neq i}^N$, and decodes the messages into $q_{t-1,i}=f_i^r\left(\boldsymbol{m}_{t-1,(-i) i}\mid \theta_{t,i}^r\right)$;
                    \STATE Gets the action $a_{t,i}=\boldsymbol{\pi}_i(\cdot\mid o_{t,i},q_{t-1,i})$ from the actor $\boldsymbol{\pi}_i$ of agent $i$;
                    \STATE \text{Samples  message} $p_{t,i}$ \text{from the message distribution} $\mathcal{N}\left(\mu_{t,i}, \Sigma_{t,i}\right)$, where $\mu_{t,i}=f_i^\mu\left(o_{t,i}, a_{t,i} ; \theta_{t,i}^\mu\right)$ and $\Sigma_{t,i}=f_i^\sigma\left(o_{t,i}, a_{t,i} ; \theta_{t,i}^\sigma\right)$;
                    \STATE \text{Samples the privacy noise} $u_i \sim \mathcal{N}\left(0, \sigma_i^2 \mathbf{I}_d\right)$;
                    \STATE Sends privatized message $m_{t,i}=p_{t,i}+u_{t,i}$ to all other agents;
                    \STATE Executes action $a_{t,i}$ and receives reward $r_{t,i}$;
                    \STATE Adds $(o_{t,i},a_{t,i},r_{t,i},\boldsymbol{m}_{t-1,(-i) i})$ to replay buffer $\mathcal{D}_i$;
                \ENDFOR
            \ENDFOR
            \STATE Update the critic, actor, message sender, and message receiver using a minibatch of samples from the replay buffer $\mathcal{D}_i$ for each agent $i=1,\ldots,N$;
        \ENDFOR
    \end{algorithmic}
\end{algorithm*}

\section{Privacy Analysis}
{\revise In this section, we first present the proof of Theorem \ref{theorem: privacy for dpmac}, which guarantees the $(\epsilon_i,\delta)$-DP for communication at each step. Then we give Corollary \ref{theorem:multi_step_privacy}, which provides episode-level $(\epsilon_i,\delta)$-DP guarantee for communication, together with its proof.}
\subsection{Proof of Theorem \ref{theorem: privacy for dpmac}}
{\revise
We first present some necessary definitions and lemmas.
We start by introducing R{\'e}nyi differential privacy (RDP) and $\ell_2$-sensitivity.
}
\begin{definition}[R{\'e}nyi differential privacy, \cite{mironov2017renyi}]
\label{def: RDP}
For $\alpha>1$ and $\rho>0$, a randomized mechanism $f: \mathcal{D} \to$ $\mathcal{Y}$ is said to have $\rho$-Rényi differential privacy of order $\alpha$, or $(\alpha, \rho)$-RDP for short, if for any neighbouring datasets $D, D^{\prime} \in \mathcal{D}$ differing by one element, it holds that $D_{\alpha}\left(f(D) \| f\left(D^{\prime}\right)\right):=\log \mathbb{E}\left(f(D) /f\left(D^{\prime}\right)\right)^{\alpha} /(\alpha-1) \leq \rho$.
\end{definition}
\begin{definition}[$\ell_2$-sensitivity, \cite{dwork2014algorithmic}]
\label{def: l2-sensitivity}
The $\ell_2$-sensitivity $\Delta (q) $ of a function $q$ is defined as $\Delta (q) =\sup _{D, D^{\prime}}\left\|q (D) -q\left (D'\right) \right\|_2$, 
for any two neighbouring datasets  $D, D^{\prime} \in$ $\mathcal{D}$ differing by one element.
\end{definition}

\label{section: privacy guarantee for DPMAC}
With the properly added Gaussian noises, Lemma \ref{lem:rdp_amplification} shows that the Gaussian mechanism can satisfy RDP.
\begin{lemma}[Lemma 3.7, \cite{wang2019efficient}]\label{lem:rdp_amplification}
For function $q: \mathcal{S}^{n} \rightarrow \mathcal{Y}$, the Gaussian mechanism $\mathcal{M}=q(S)+\mathbf{u}$ with $\mathbf{u} \sim N\left(0, \sigma^{2} \mathbf{I}\right)$ satisfies $\left(\alpha, \alpha \Delta^{2}(q) /\left(2 \sigma^{2}\right)\right)$-RDP. Additionally, 
if $\mathcal{M}$ is applied to a subset of all the samples which are uniformly sampled from the whole datasets without replacement using sampling rate $\gamma$, then $\mathcal{M}$ satisfies $\left(\alpha, 3.5 \gamma^{2} \Delta^{2}(q) \alpha / \sigma^{2}\right)$-RDP with 
$\sigma^{\prime 2}=\sigma^{2} / \Delta^{2}(q) \geq 0.7$ and 
$\alpha \leq 2 \sigma^{\prime 2} \log \left(1 / \gamma \alpha\left(1+\sigma^{\prime 2}\right)\right) / 3+1$.
\end{lemma}
{\revise 
We now present the following two propositions regarding RDP. The first proposition shows that a composition of $k$ mechanisms satisfying RDP is also a mechanism that satisfies RDP. }
\begin{proposition}[Proposition 1, \cite{mironov2017renyi}]\label{prop:prop1}
If $k$ randomized mechanisms $f_i:\mathcal{D}\rightarrow \mathcal{Y}$ for all $i\in[k]$, satisfy $(\alpha,\rho_i)$-RDP,
then their composition $(f_1(D),\ldots,f_k(D))$ satisfies $(\alpha,\sum^k_{i=1}\rho_i)$-RDP.
\end{proposition}
{\revise The second proposition provides the transformation from an RDP guarantee to a corresponding DP guarantee.}
\begin{proposition}[Proposition 3, \cite{mironov2017renyi}]\label{prop:prop2}
If a randomized mechanism $f:\mathcal{D}\rightarrow \mathcal{Y}$ satisfies $(\alpha,\rho)$-RDP, then $f$
satisfies $(\rho+\log (1 / \delta) /(\alpha-1), \delta)$-DP for all $\delta\in(0,1)$.
\end{proposition}



We are now ready to give the formal proof of Theorem \ref{theorem: privacy for dpmac}. 
Recall $p_i=f^s_i(o_i,a_i;\theta^s_i)\sim\mathcal{N}(\mu_i,\Sigma_i)$ is the generated stochastic message before injecting privacy noise $u_{i} \sim \mathcal{N}\left(0, \sigma_{i}^{2} \mathbf{I}_{d}\right)$, and $m_{i}=p_{i}+u_{i}$ is the privatized message to be sent by agent $i$.
Here we drop the dependency of the subscript of message on the index of the target agent (\textit{i.e.}, we abbreviate the message $m_{i,j}$ sent from agent $i$ to agent $j$ as $m_i$) when it is clear from the context. 
We slightly abuse the notation by writing $f^s_i=f^s_i(o_i,a_i;\theta^s_i)$.
To bound the sensitivity of $f^s_i$, one can perform the norm clipping to restrict the $\ell_2$ norm of $f^s_i$ by replacing $ f^s_i $ with $ f^s_i/\max(1,\| f^s_i\|_2/C)$, which ensures that $\| f^s_i\|_2\leq C$. At each time $t$, for agent $i$, each message function $f_i^s$ is applied to a subset of transitions in local trajectory $\tau_{i}$ of agent $i$ using uniform sampling without replacement with sampling rate $\gamma_1$, and agent $i$ samples a subset of target agents to send messages by uniform sampling without replacement with sampling rate $\gamma_2$.
\begin{proof}[Proof of Theorem \ref{theorem: privacy for dpmac}]
Due to the norm clipping of $ f_i^s$ and the triangle inequality, the $\ell_2$-sensitivity of $ f_i^s$ could be bounded as
\begin{align}\label{eq:pf_thm5.1}
    \Delta_2( f_i^s)&=\sup_{D,D^\prime}\| f_i^s(D)- f_i^s(D^\prime)\|_2\notag\\
    &\leq 2C\,,
\end{align}
where  $D, D^{\prime} \in$ $\mathcal{D}$ are any two neighbouring datasets differing by one element.
By Equation (\ref{eq:pf_thm5.1}) and
Lemma \ref{lem:rdp_amplification},
the privatized message $m_i$
satisfies 
$\left(\alpha, 14 \gamma_1^{2} C^2 \alpha / \sigma^{2}\right)$-RDP. 
Since agent $i$ samples $\gamma_2N$ target agents to communicate, all the messages sent by agent $i$ at time $t$ are actually sent by a composite message function
$M_i=\{m_{i,k_j}\}^{\gamma_2N}_{j=1}$, 
which satisfies $(\alpha,14\gamma_2\gamma_1^2NC^2\alpha/\sigma^2)$-RDP by Proposition \ref{prop:prop1}. 
Substituting $\sigma^2_i=\frac{14\gamma_2\gamma_1^2NC^2\alpha}{\beta\epsilon_i}=\frac{14\gamma_2\gamma_1^2NC^2\alpha}{\epsilon_i+\frac{\log\delta}{\alpha-1}}$ shows that $M_{i}$ satisfies
$(\alpha,\epsilon_i+\frac{\log \delta}{\alpha-1})$-RDP.
Applying Proposition \ref{prop:prop2} shows that
$M_i$ satisfies $(\epsilon_i,\delta)$-DP, which concludes the proof.
\end{proof}

{\revise 
\subsection{Analysis of Episode-level $(\epsilon_i,\delta)$-DP}
We would like to emphasize that Theorem \ref{theorem: privacy for dpmac} also provides an episode-level privacy guarantee. To see this, consider a multi-step game with finite episode length of $T$. Since Theorem \ref{theorem: privacy for dpmac} guarantees the $(\epsilon_i,\delta)$-DP for the communication mechanism $M^t_i$ in each step $t$, the communication mechanism of the whole episode $M_i=\{M^t_i\}^T_{t=1}$ satisfies $(T\epsilon_i+(T-1)\frac{\log \delta}{\alpha-1},\delta)$-DP based on Proposition \ref{prop:prop1} and Proposition \ref{prop:prop2}. 
Further, we can attain the $(\epsilon_i,\delta)$-DP of the communication of DPMAC in the whole episode by adjusting the noise variance in Theorem \ref{theorem: privacy for dpmac}, detailed in the following corollary.
\begin{corollary}[Episode-level $(\epsilon_i,\delta)$-DP guarantee for DPMAC]
\label{theorem:multi_step_privacy}
Consider an episode with finite length $T$.
Let $\gamma_1,\gamma_2\in (0,1)$ having the same definitions as in Theorem \ref{theorem: privacy for dpmac}, and let $C$ be the sensitivity of the message functions.
For any $\delta>0$ and privacy budget $\epsilon_i$, the communication of agent $i$ in the whole episode satisfies $(\epsilon_i,\delta)$-DP when 
$\sigma^2_i=\frac{14\gamma_2\gamma_1^2NC^2\alpha T}{\beta\epsilon_i}$,
if we have 
$\alpha = \frac{\log \delta^{-1}}{\epsilon_i(1-\beta)}+1 \leq 2 \sigma^{\prime 2} \log \left(1 / \gamma_1 \alpha\left(1+\sigma^{\prime 2}\right)\right) / 3+1$ with $\beta\in (0,1)$ and $\sigma^{\prime 2}=\sigma^2_i/(4C^2)\geq 0.7$.
\end{corollary}
\begin{proof}
Analogous to the proof of Theorem \ref{theorem: privacy for dpmac}, by Equation (\ref{eq:pf_thm5.1}) and Lemma \ref{lem:rdp_amplification}, the privatized message $m^t_i$ at time $t$ satisfies $\left(\alpha, 14 \gamma_1^{2} C^2 \alpha T/ \sigma^{2}\right)$-RDP. 
Similarly, since agent $i$ samples $\gamma_2N$ target agents to communicate, all the messages sent by agent $i$ at time $t$ are through a composite message function $\cM^t_i=\{m^t_{i,k_j}\}^{\gamma_2N}_{j=1}$. By Proposition \ref{prop:prop1}, $\cM^t_i$ satisfies $(\alpha,14\gamma_2\gamma_1^2NC^2\alpha T/\sigma^2)$-RDP. 
Substituting $\sigma^2_i=\frac{14\gamma_2\gamma_1^2NC^2\alpha T}{\beta\epsilon_i}=\frac{14\gamma_2\gamma_1^2NC^2\alpha T}{\epsilon_i+\frac{\log\delta}{\alpha-1}}$, we have $\cM^t_{i}$ satisfies $(\alpha,(\epsilon_i+\frac{\log \delta}{\alpha-1})/T)$-RDP. 
Then applying Proposition \ref{prop:prop1} again shows that the composite message mechanism $\cM_i=\{\cM^t_i\}^T_{t=1}$ is $(\alpha,\epsilon_i+\frac{\log \delta}{\alpha-1})$-RDP. As all the messages sent by agent $i$ in the whole episode are through message mechanism $\cM_i$, the privacy of these messages is strictly protected.
Lastly, the proof is completed by applying Proposition \ref{prop:prop2} to translate $(\alpha,\epsilon_i+\frac{\log \delta}{\alpha-1})$-RDP to $(\epsilon_i,\delta)$-DP for mechanism $\cM_i$.
\end{proof}
}

\section{Equilibrium Analysis}
\subsection{Single-Step Game}
\label{appendix: single step game in appendix}
In this section, we give the detailed analysis of our single-step game. 
We first give the notion of the potential game (PG) as follows.
\begin{definition}[Potential game, \cite{monderer1996potential}] \label{def: potential game}

A two-player game $G$ is a potential game if the mixed second order partial derivative of the utility functions are equal:
\begin{align*}
\partial_{p_{1}} \partial_{p_{2}} u_{1}=\partial_{p_{1}} \partial_{p_{2}} u_{2}\,.
\end{align*}
\end{definition} 

The intuition behind the PG is that it tracks the changes in the
payoff when some player deviates, without taking into account which one. Thus the PG usually helps with the analysis of the cooperative game, where the players might have the similar potential to act. To analyze the games involving  multi-agent coordination with state dependence, the Markov potential game (MPG) is recently studied \citep{DBLP:conf/iclr/MacuaZZ18,leonardos2021global}, where the action potential of all agents is described by a potential function.
The solution concept of the PG relies on Nash equilibrium (NE) \citep{nash1950equilibrium}, the existence of which guarantees that all the agents could act in the best response to others.


We present the definition of the privacy loss function $c^{\cM}(p_n)$ in Definition \ref{def: privacy loss function}.
Recall $\cM$ is the privacy preserving mechanism used in the game.
\begin{definition}[Privacy loss function, \cite{DBLP:journals/popets/PejoTB19}]
\label{def: privacy loss function} The privacy loss function $c:[0,1] \rightarrow[0,1]$ is a continuous and twice differentiable function with $c(0)=$ $1, c(1)=0$ and $\partial_{p_{n}} c<0$.
\end{definition}

Then the benefit function $b\left(V_n, V^{\cM}_n(p_1,p_2) \right)$ is given in Definition \ref{def: benefit function}. \citet{DBLP:journals/popets/PejoTB19} consider the negative training error as the benefit while ours is different, as introduced here. $V_n$ is the value that agent $n$ receives when agent $n$ acts alone without any cooperation. $V_n(p_1,p_2)$ is the value that agent $n$ receives when agent $n$ acts cooperatively under the privacy mechanism $\cM$ and the actions $p_1$ and $p_2$.  Intuitively, we want that the benefit function portrays the benefit of the cooperative actions over the non-cooperative ones, thus in the meaningless case $V_n \geq V_{n}^{\cM}$, $b\left(V_n, V^{\cM}_n(p_1,p_2) \right)=0$. In addition, $\partial_{p_{n}} b \leq 0$ since the value of the benefit function should decrease as the level of privacy protection increases.

\begin{definition}[Benefit function]\label{def: benefit function} The benefit function
$b: \mathbb{R}^{+} \times \mathbb{R}^{+} \rightarrow \mathbb{R}_{0}^{+}$ is a continuous and twice differentiable function, with $\partial_{p_{n}} b \leq 0$ and $b\left(V_n, V^{\cM}_n(p_1,p_2) \right)=0$ if $V_{n} \geq V_{n}^{\cM}$.
\end{definition}

Finally, we give the definition of the value function $V_n$ and $V_n^{\cM}(p_1,p_2)$. $V_n$ can be viewed as the upper bound of $V_n^{\cM}(p_1,p_2)$, which intuitively means the value for the pure cooperation without privacy protection is always larger than that with some privacy protection. 
Note that the value function defined here is not the one which is commonly defined in RL literature. 
\begin{definition}[Value function]\label{def: value function in CGP}
$V_n$ is the value function for agent $n$ acting alone and without any cooperation. $V_{n}^{\cM}:[0,1] \times[0,1] \rightarrow \mathbb{R}^{+}$ is continuous, twice differentiable and:
\begin{itemize}
    \item $\exists m \in \mathcal{N}: p_{m}=1 \Rightarrow \forall n \in \mathcal{N}: V_{n}^{\cM}\left(p_{1}, p_{2}\right) \leq V_{n}$,
    \item $\forall n, m \in \mathcal{N}: \partial_{p_{m}} V_{n}^{\cM}<0$,
    \item $\forall n \in \mathcal{N}: V_{n}<V_{n}^{\cM}(0,0)$.
\end{itemize}
\end{definition}
For $V_n^{\cM}$, the first condition ensures that if one agent protects the privacy entirely, the total value will not be greater than the value that agents act alone. 
The second condition constrains that the value function $V_n^{\cM}$ should be negatively monotonic w.r.t. $p_m$ since another agent's privacy gets stronger, which thus leads to less corporation. 
The third condition tells that the total value without any privacy protection (\textit{i.e.}, pure corporation) should be certainly larger than the value that agents act alone.

The following theorem shows that the PG enjoys a good property that the NE exists. This motivates us to {\revise translate} CGP into a potential game, which will {\revise be the key to prove} Theorem \ref{theorem: NE in CGP}.
\begin{theorem}[Monderer and Shapley, \cite{monderer1996potential}]
\label{theorem: NE}
The potential game has at least one pure-strategy NE.
\end{theorem}
We are now ready to prove Theorem \ref{theorem: NE in CGP}.
\begin{proof}[Proof of Theorem \ref{theorem: NE in CGP}]
Starting from the definition of the two-player potential game (Definition \ref{def: potential game}) and the utility function (Definition \ref{def: CGP}), we have
\begin{align*}
    \partial_{p_{1}} \partial_{p_{2}} u_{1}=\partial_{p_{1}} \partial_{p_{2}} u_{2}\,,
\end{align*}
which implies that 
\begin{align*}
    \partial_{p_{1}}\partial_{p_{2}} b\left(V_{1}, V_{1}^{\cM}\left(p_{1}, p_{2}\right)\right)
=\partial_{p_{1}} \partial_{p_{2}} b\left(V_{2}, V_{2}^{\cM}\left(p_{1}, p_{2}\right)\right)\,.
\end{align*}
Hence, due to the chain rule, it holds that 
\begin{align}\label{eq: induction 1}
    &(\partial_{ V_{n}^{\cM} }^{2} b) \cdot\left(\partial_{p_{1}} V^{\cM}_{1}-\partial_{p_{2}} V^{\cM}_{2}\right)\notag\\
    =&(\partial_{ V_{n}^{\cM} } b) \cdot\left(\partial_{p_{1}} \partial_{p_{2}} V^{\cM}_{2} -\partial_{p_{1}} \partial_{p_{2}} V^{\cM}_{1}\right)\,.
\end{align}
If the value function is with the property $\partial_{p_{1}}^{i} V_{1}=\partial_{p_{2}}^{i} V_{2}$, $\forall i \in \{1,2\}$,
CGP will satisfy Equation \eqref{eq: induction 1}. Further, by Theorem \ref{theorem: mpg NE}, CGP is a potential game and has at least one non-trivial pure-strategy Nash equilibrium, thus concluding the proof.
\end{proof}

Utilizing Theorem \ref{theorem: NE in CGP} and designing the necessary functions for CGP, we show that single round binary sums could be proved with the existence of NE.
\begin{theorem}[NE guarantee in single round binary sums]
\label{theorem: single round binary sums}
For the single round binary sums game, let 
$N=2$ be the number of players, $p_n$ be the probability of the randomized response, $c^{\cM}(p_n)=1-p_n$ be the privacy loss function, $V_n=-\frac{1}{2}$ and $V^{\cM}_n(p_1,p_2)=-\frac{1}{2}p_1^2-\frac{1}{2}p_2^2 -p_1p_2=-\frac{1}{2}(p_1+p_2)^2$ be the value function, $b(V_n, V_n^{\cM}(p_1,p_2))=V^{\cM}_n(p_1,p_2)-V_n=-\frac{1}{2}p_1^2-\frac{1}{2}p_2^2-p_1p_2+\frac{1}{2}$ be the benefit function,   $u_{n}\left(p_{1}, p_{2}\right)=B_{n} \cdot b\left( V_n, V^{\cM}_n(p_1,p_2)   \right)-C_{n}^{M} \cdot c^{M}\left(p_{n}\right)$ be the utility function, then Theorem \ref{theorem: NE in CGP} holds, which is to say, single round binary sums can be formulated into a CGP, further leading to the existence of one non-trivial pure-strategy NE.
Further, the utility function is
\begin{align}
\label{eq: utility function in single round binary sums}
u_n = -\frac{B_n}{2}(p_1+p_2)^2 + C_n p_n +  \frac{B_n}{2} - C_n,\quad \forall n \in \{1,2\}\,.
\end{align}
The strategy taken by the agent is
\begin{align*}
p_n=\argmax_{p_n} u_n = \frac{C_n}{B_n} - p_{-n},\quad \forall n \in \{1,2\} \,.
\end{align*}
\end{theorem}

\subsection{Multi-Step Game}
\label{appendix: multi-step game}
{\revise In this section, we present the equilibrium analysis of the multi-step game. We start by presenting the definition of MPG.}
\begin{definition}[Markov potential game, \cite{leonardos2021global}] 
\label{def: markov potential game}
A Markov decision process (MDP), $\mathcal{M}$, is called a Markov Potential Game $(M P G)$ if there exists a state-dependent function $\Phi_{s}: \Pi \rightarrow \mathbb{R}$ for $s \in \mathcal{S}$ such that
\begin{align*}
    \Phi_{s}\left(\pi_{i}, \pi_{-i}\right)-\Phi_{s}\left(\pi_{i}^{\prime}, \pi_{-i}\right)=V_{s}^{i}\left(\pi_{i}, \pi_{-i}\right)-V_{s}^{i}\left(\pi_{i}^{\prime}, \pi_{-i}\right)
\end{align*}
holds for all agents $i \in \mathcal{N}$, all states $s \in \mathcal{S}$ and all policies $\pi_{i}, \pi_{i}^{\prime} \in \Pi_{i}, \pi_{-i} \in \Pi_{-i}$. By linearity of expectation, it follows that $\Phi_{\rho}\left(\pi_{i}, \pi_{-i}\right)-\Phi_{\rho}\left(\pi_{i}^{\prime}, \pi_{-i}\right)=V_{\rho}^{i}\left(\pi_{i}, \pi_{-i}\right)-V_{\rho}^{i}\left(\pi_{i}^{\prime}, \pi_{-i}\right)$, where $\Phi_{\rho}(\pi):=\mathbb{E}_{s \sim \rho}\left[\Phi_{s}(\pi)\right] .$
\end{definition}

The solution concept of the PG and the MPG
relies on NE \citep{nash1950equilibrium}, the existence of which guarantees that all the agents can act 
in the best response to others. 
Formally, the 
best response
of agent $i$ with respect to the opponent’s policy $\pi^{-i}_{\theta^{-i}}$ indicates the policy $\pi^i_{\ast}$ such that $V^{i}\left(s ; \pi_{*}^{i}, \pi_{\theta^{-i}}^{-i}\right) \geq V^{i}\left(s ; \pi_{\theta^{i}}^{i}, \pi_{\theta-i}^{-i}\right)$ for all feasible $\pi^i_{\theta^i}$.
While it is still not clear how privacy will influence the equilibrium in the MARL setting, \citet{kumari2016cooperative} study the cooperative game with privacy and \citet{DBLP:journals/popets/PejoTB19}  model the game with privacy to solve the private learning problem.

Now we introduce the following assumptions and Theorem \ref{theorem: mpg NE} \citep{DBLP:conf/iclr/MacuaZZ18} to support the main theorems.
\begin{assumption}
\label{asmp: 1}
The state and parameter sets, $\mathbb{Z}$ and $\mathbb{W}$, are nonempty and convex.
\end{assumption}

We slightly abuse the notation in Assumption \ref{asmp: 2} and use $\sigma_{k,i}$ as a given random variable with distribution $p_{\sigma_k}(\cdot | x_i,a_i)$ instead of the standard deviation of the noise.

\begin{assumption}
\label{asmp: 2}
The reward functions $r_{k}(x_i,a_i,\sigma_{k,i})$ are twice continuously differentiable in $\mathbb{Z} \times \mathbb{W}, \forall k \in \mathcal{N}$.
\end{assumption} 

\begin{assumption}
\label{asmp: 3}
The state-transition function, $f$, and constraints, $g$, are continuously differentiable in $\mathbb{Z} \times \mathbb{W}$, and satisfy some regularity conditions (e.g., Mangasarian-Fromovitz).
\end{assumption}

\begin{assumption}
\label{asmp: 4}
The reward functions $r_{k}$ are proper, and there exists a scalar $B$ such that the level sets $\left\{a_{0} \in \mathbb{C}_{0},\left(x_{i}, a_{i}\right) \in \mathbb{C}_{i}: \mathbb{E}\left[r_{k}\left(x_{i}, a_{i}, \sigma_{k, i}\right)\right] \geq B\right\}_{i=0}^{\infty}$ are nonempty and bounded $\forall k \in \mathcal{N}$.
\end{assumption}

\begin{theorem}[\citep{DBLP:conf/iclr/MacuaZZ18}]
\label{theorem: mpg NE}
Let Assumptions \ref{asmp: 1}, \ref{asmp: 2}, \ref{asmp: 3}, \ref{asmp: 4} hold. Then, the game in Equation (8) of \cite{DBLP:conf/iclr/MacuaZZ18} is an MPG if and only if: i) the reward function of every agent can be expressed as the sum of a term common to all agents plus another term that depends neither on its own state-component vector nor on its policy parameter: 
\begin{align}
\label{eq: ref 1}
\begin{split}
&r_{k}\left(x_{k, i}^{r}, \pi_{k}\left(x_{k, i}^{\pi}, w_{k}\right), \pi_{-k}\left(x_{-k, i}^{\pi}, w_{-k}\right), \sigma_{k, i}\right)\\
=&J\left(x_{i}, \pi\left(x_{i}, w\right), \sigma_{i}\right)+\Theta_{k}\left(x_{-k, i}^{r}, \pi_{-k}\left(x_{-k, i}^{\pi}, w_{-k}\right), \sigma_{i}\right)\,,
\end{split}
\end{align}
$\forall k \in \mathcal{N}$; and ii) the following condition on the non-common term holds:
\begin{align}
\label{eq: ref 2}
\mathbb{E}\left[\nabla_{x_{k, i}^{\Theta}} \Theta_{k}\left(x_{-k, i}^{r}, \pi_{-k}\left(x_{-k, i}^{\pi}, w_{-k}\right), \boldsymbol{\sigma}_{i}\right)\right]=0\,,
\end{align}
where $x_{k, i}^{\Theta} \triangleq\left(x_i(m)\right)_{m \in \mathcal{X}_k^{\Theta}}$, $\mathcal{X}_k^{\Theta}=\mathcal{X}_k^\pi \cup \mathcal{X}_k^r$, $\mathcal{X}_k^\pi$ is the set of state vector components
that influence the policy of agent $k$ and $\mathcal{X}_k^r$ is the set of components of the state vector that influence the reward of agent $k$ directly.
Moreover, if Equation (\ref{eq: ref 2}) holds, then the common term in Equation (\ref{eq: ref 1}), $J$, equals the potential function.

\end{theorem}

To make multiple round sums an MPG, the following detailed settings including the state space, the state transition, the action space, and the reward function are given. By such a specific design, we could achieve Theorem \ref{theorem: NE Guarantee in Multiple Round Sums in appendix}, which shows the existence of a NE.

    \paragraph{State}
    The state $x_{i,t} \in \mathbb{R}$ represents the remaining saving of agent $i$ at time $t$ (initially $x_{i,0}$).
    
    \paragraph{State transition}
    The transition function is deterministic, which is $x_{i,t+1} = x_{i,t} - b_{i,t}$.

     \paragraph{Action}
     For $i\in [N]$, at time $t$, agent $i$ with policy $\pi_i$ assigns $b_{i,t}$ savings according to the current state and the message sent by other agents. Then agent $i$ performs the randomized perturbation on the assignment with privacy level $p_{i,t}$, with the message sender protocol $f_i^s$. After that the message $m_{i,t}$ is generated and sent to other agents. Formally,
     \begin{align}
     \begin{split}
    b_{i,t},p_{i,t}&=\pi_i^b(x_{i,t}, m_{-i, t-1}),\pi_i^p(x_{i,t}, m_{-i, t-1})\,, \notag \\
    m_{i,t} &= {f}_i^s(x_{i,t}, b_{i,t},p_{i,t})\,.\notag 
    \end{split}
    \end{align}
    
     \paragraph{Reward function}
     The reward function for agent $i$ is designed with the trade-off between privacy and utility, as shown in Equation (\ref{eq: multi-step game reward}). Privacy reward is performed to praise the privacy preserving. Formally, $\forall i\in N$,
    \begin{align}
    \label{eq: multi-step game reward}
     &r_i(x_{i,t}, \pi_i(x_{i,t}), \pi_{-i}(x_{-i,t}))\notag\\
    =& \sum_{j\in N} (1-p_{j,t})b_{j,t} + \alpha x_{i,t} + \beta p_{i,t}\,.
    \end{align}

The reward function given in Equation (\ref{eq: multi-step game reward}) could be written in two terms respectively, where $ \sum_{j\in N} (1-p_{j,t})b_{j,t}$ represents the first term in Equation (\ref{eq: ref 1}) and $\alpha x_{i,t} + \beta p_{i,t}$ represents the second term. Theorem \ref{theorem: NE Guarantee in Multiple Round Sums in appendix} is thus given.
\begin{theorem}[NE guarantee in multiple round sums]
\label{theorem: NE Guarantee in Multiple Round Sums in appendix}
For the multiple round sums game, the reward function in Equation (\ref{eq: multi-step game reward}) will satisfy Theorem \ref{theorem: mpg NE} if the gradient between the message $m_{-i}$ and the message function $f_{i}$ is $0$. 
Further, if Assumptions \ref{asmp: 1}, \ref{asmp: 2}, \ref{asmp: 3}, \ref{asmp: 4} (see Appendix \ref{appendix: multi-step game}) are satisfied, our MPG could find a NE with potential function $J$ as shown below,
    \begin{align}
    J(x_{t}, \pi(x_t))=\sum_{j\in [N]} \left( (1-p_{j,t})b_{j,t}+\alpha x_{j,t}+\beta p_{i,t}  \right)\,.
    \end{align}

\end{theorem}


\section{Implementation Details}\label{sec:app:implementation}
\paragraph{Agent}
We build our communication method upon MADDPG \citep{MADDPG_2017}.
However, we note that our communication method is general enough to be built on the top of any MARL algorithm with the CTDE paradigm.
Besides, we use recurrent neural networks (RNN) for the actors of agents to approximate policies, which is shown to be  effective in partially observable environments \citep{hausknecht2015deep}.

\paragraph{Communication protocol}
The sender utilizes the multi-layer perceptron (MLP) and the receiver is an attention-based  network \citep{attention_2017}. For the sender, a shared linear layer is first applied and two linear layer follows to output the mean and the logarithm of the standard deviation. For the receiver, three linear layers corresponding to $K,Q,V$ is directly applied with no shared layer. 


\paragraph{Reparameterization trick}
The stochastic message sender adopts the multivariate Gaussian distribution, which is implemented via the reparameterization trick \citep{KingmaW13}.
Specifically, in our implementation, the stochastic message $p_i$ is
\begin{align*}
    p_i=\mu_i+\tilde{\sigma}_i\odot\xi\,,
\end{align*}
where $\mu_i$ and $\tilde{\sigma}_i$ are the mean and variance as the outputs of two neural networks of the sender, $\xi\sim\mathcal{N}(\bm{0},\mathbf{I}_d)$, and $\odot$ denotes the element-wise product.



\section{Training Details}\label{sec:app:training}
\paragraph{Optimization process of message senders and receivers}
{\revise The message receiver of agent $i$ encodes all the received messages into the encoded messages, which together with the observation of agent $i$ will be fed into the actor of agent $i$. Hence the message receiver of agent $i$ serves as a component in the actor of agent $i$, which will be updated by backpropagating the actor loss using policy gradient. By the chain rule, since the messages received by the receiver of agent $i$ are sent from the message senders of all other agents, the message senders of all other agents will then be updated by backpropagation from the message receiver of agent $i$.}

\paragraph{Parameter setting}
The hyperparameters of all algorithms are shown in Table \ref{table: hyper param}. 
Since all the algorithms are built on the top of MADDPG, we tune the hyperparameters such that MADDPG has the best performance, which is also adopted in \cite{intention_sharing_2021}.
The hidden dimension for actor, critic, and the message protocol is selected by grid searching in $\{32, 64,128,256\}$, and the message dimension is selected by conducting grid search among values $\{4,8,16,32\}$.
\begin{table*}[h]
\centering
\caption{Hyperparameters of all algorithms.}
\label{table: hyper param}
\begin{tabular}{lcccc}
\hline & MADDPG & TarMAC & I2C & DPMAC \\
\hline
Discount Factor & $0.99$ & $0.99$ & $0.99$ & $0.99$ \\
Batch Size (CCN\&CC) & $128$ & $128$ & $128$ & $128$ \\
Batch Size (PP) & $256$ & $256$ & $256$ & $256$ \\
Buffer Size  & $1\times10^4$ & $1\times10^4$ & $1\times10^4$ & $1\times10^4$ \\
Optimizer & Adam & Adam & Adam & Adam \\
Activation Function & ReLU & ReLU & ReLU & ReLU \\
Learning Rate & $7\times10^{-4}$ & $7\times10^{-4}$  & $7\times10^{-4}$  & $7\times10^{-4}$ \\
Hidden Dimension for Actor/Critic & $128$ & $128$ & $128$ & $128$\\
Hidden Dimension for Message Protocol & $-$ & $32$ & $32$ & $32$\\
Message Dimension & $-$ & $8$ & $8$ & $8$ \\
\hline
\end{tabular}
\end{table*}

\section{Environment Details}
\label{appendix: environment details}
\paragraph{Cooperative navigation (CN)}
Cooperative navigation is a standard task for multi-agent systems, introduced in \cite{MADDPG_2017}, where agents target at reaching their own landmarks while avoiding collision. There are in total $N=3$ agents for our experiment setting. 

\paragraph{Cooperative communication and navigation (CCN)}
The cooperative communication task is introduced in \cite{mordatch2017emergence}, where agents' goal is to reach their target landmark while each agent only knows the location of the other agent's target. The agents do not communicate with the channels embedded in the task and only share personal information with a learned protocol.

\paragraph{Predator prey (PP)}
The predator prey task is a standard task and we use the same setting and the same evaluation way as in \cite{I2C_2020}.  Specifically, there are $N=3$ predators and $M=2$ preys in this task, whose initial positions are randomized initialized. Each predator is controlled by a agent, 
and each prey moves in the closest predator's opposite direction.
Since the speed of preys is higher than that of predators, cooperation  is required for agents to capture a prey. The team reward is the negative sum of physical distances from all the predators to their closet preys.
Besides, the predators will be penalized for each time any two predators collide with each other and we set the collision penalty as $r_{\operatorname{collision}}=-1$. Each episode has $40$ timesteps in this task.

\section{Additional Experiment Results}
\label{sec: supplementary materials for Experiments}
In this section, we present some additional experiment results in Figure \ref{fig: dp tarmac}, and Figure \ref{fig: dp i2c}.
\begin{figure*}[!thbp]
    \centering
    \includegraphics[width=1.0\textwidth]{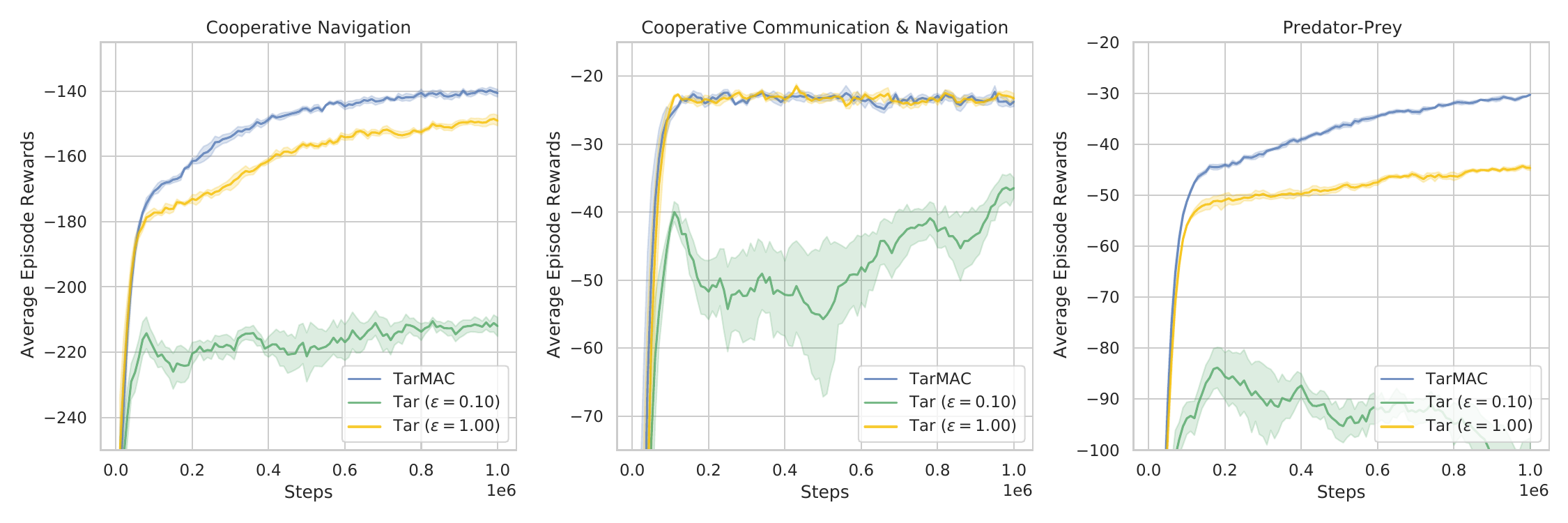}
    \caption{Performance of TarMAC with $\epsilon=0.1$ and $\epsilon=1.0$.}
    \label{fig: dp tarmac}
\end{figure*}

\begin{figure*}[!thbp]
    \centering
    \includegraphics[width=1.0\textwidth]{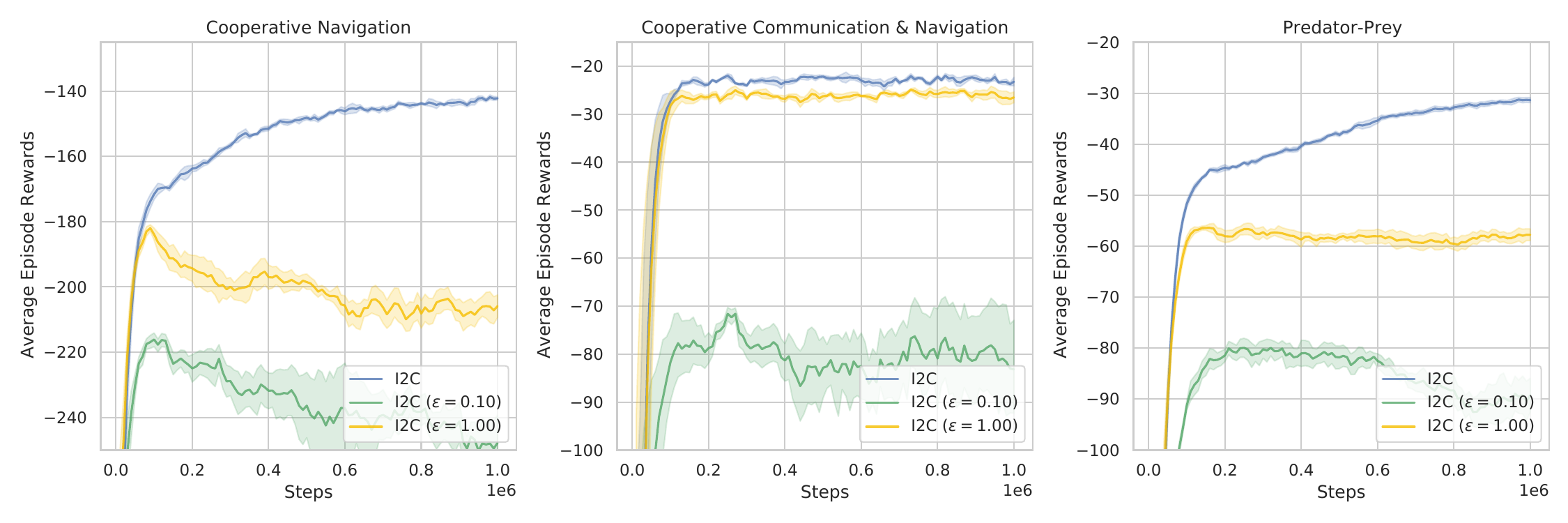}
    \caption{Performance of I2C with $\epsilon=0.1$ and $\epsilon=1.0$.}
    \label{fig: dp i2c}
\end{figure*}

{\revise 
\section{Additional Experiment Analyses}\label{sec:app_addition_exp_analyses}
In this section, we present the additional experiment analyses regarding Figure \ref{fig:fig_single_step_exp}.
\subsection{Experiment Results without Privacy}\label{sec:further_analyses_wo_privacy}
We first analyze the baseline methods, TarMAC and I2C, which are effective communication-based algorithm and can achieve superior performance than methods without communication such as MADDPG. As shown in Figure \ref{fig: communication exp}, TarMAC and I2C outperform MADDPG on cooperative navigation (CN) and predator-prey (PP) tasks and achieve similar performance as MADDPG on cooperative communication and navigation (CCN) task. However, DPMAC without privacy constraints outperforms TarMAC and I2C by a large margin on CCN task and is comparable with TarMAC and I2C on CN and PP tasks. 
These results demonstrate the effectiveness of the message sender and receiver of DPMAC, which also has competitive performance under no privacy constraints.

\subsection{Experiment Results with Privacy}\label{sec:further_analyses_w_privacy}
We further investigate how the DP noise affects our DPMAC and other baseline methods. From Figure \ref{fig: privact exp 0.1}, one can observe a serious performance degradation of TarMAC and I2C with a privacy budget of $\epsilon=0.10$, which makes these two methods even worse than non-communicative MADDPG. In contrast, our DPMAC could still significantly outperform MADDPG on CN and PP tasks and is comparable with MADDPG on CCN task. From Figure \ref{fig: privacy exp 1.0}, which is with a privacy budget of $\epsilon=1.00$ (note that $\epsilon=1.00$ enjoys weaker privacy guarantee than $\epsilon=0.10$ due to fewer privacy noises injected), our DPMAC still clearly outperforms all other methods on CN and PP tasks and is comparable with MADDPG on CCN task. In sharp contrast to DPMAC, TarMAC still has clear degeneration on all tasks and is even worse than MADDPG on PP task, and I2C fails across all three tasks under such privacy constraints.
These results demonstrate that the design principle of DPMAC that adjusting the learned message distribution via
incorporating the distribution of privacy noise into the optimization process of the message sender and receiver can well alleviate the negative impacts of the privacy noises.

\subsection{DPMAC under Different Privacy Budgets}\label{sec:further_analyses_on_different_privacy}
In Figure \ref{fig: privacy exp dpmac}, we further show the performance of DPMAC under different privacy budget $\epsilon$. On CN task, the performance of DPMAC with different privacy budgets are very close except $\epsilon=0.01$, which is a very stringent privacy level. 
On CCN task, DPMAC without privacy constraint achieves the best performance, and DPMAC with different $\epsilon$ could also reach meaningful performance. 
On PP task, DPMAC with $\epsilon=0.10$ and $\epsilon=1.00$ even outperform non-private DPMAC. 
As we have analyzed in Section \ref{sec:method_sender}, this is because DPMAC can
learn the parameter ${\theta}_i^s$ of the stochastic message sender such that ${D}_{\operatorname{KL}}(\mathcal{N}(\mu_i, \Sigma_i+\sigma^2_i\mathbf{I}_d) \| \mathcal{N}(\mu^\ast_i,\Sigma^\ast_i))\leq {D}_{\operatorname{KL}}(\mathcal{N}(\mu^\prime_i, \Sigma^\prime_i) \| \mathcal{N}(\mu^\ast_i,\Sigma^\ast_i))$, which means that
the private messages can even encourage better team cooperation to gain higher team rewards than the non-private messages.
Overall, one can see that the performance of DPMAC drops clearly only when $\epsilon=0.01$ on PP task. 
The above results also clearly demonstrate that DPMAC can learn to adjust the message distribution to alleviate the potential negative impacts of privacy noises, which ensures meaningful performance even under limited privacy budgets and stabilizes performance.

}

{\revise \section{Experiments with Episode-level $(\epsilon_i,\delta)$-DP}
In this section, we present the experiment results  
under episode-level $(\epsilon_i,\delta)$-DP constraints 
with the privacy noise specified in Theorem \ref{theorem:multi_step_privacy}  and the corresponding experiment analyses.

\subsection{Experiment Results with Episode-level $(\epsilon_i,\delta)$-DP}
Comparing Figure \ref{fig: privact exp 0.1} and \ref{fig: episodic privact exp 0.1}, 
under episode-level privacy constraint of $\epsilon=0.10$, both the performance of TarMAC and I2C has a further degeneration across all the tasks. However, DPMAC could still clearly outperform MADDPG on CN task and is comparable with MADDPG on CCN task.
On PP task, DPMAC is outperformed by MADDPG but still exceeds TarMAC and I2C by a large margin. 
Comparing Figure \ref{fig: privacy exp 1.0} and \ref{fig: episodic privacy exp 1.0}, I2C still fails across all tasks and TarMAC also has a clear performance drop. Specifically, 
TarMAC has exceeded MADDPG on CN task and was comparable with MADDPG on CCN task in Figure \ref{fig: privacy exp 1.0}, but now TarMAC turns out be outperformed by MADDPG on these tasks in Figure \ref{fig: episodic privacy exp 1.0}. Besides, the performance gap between MADDPG and TarMAC now also becomes larger under the episode-level privacy constraint on the PP task. 
However, even though under episode-level privacy constraint of $\epsilon=1.00$, DPMAC still outperforms MADDPG on CN and PP tasks and is comparable with MADDPG on the CCN task. 
The above results indeed demonstrate that it is harder to learn under the episode-level privacy constraints since larger privacy noises are injected into the messages, but also validate the effectiveness  of DPMAC, which still has good performance under episode-level privacy constraints.

\subsection{DPMAC under Different Episode-level Privacy Budgets}
We now analyze the performance of DPMAC under different episode-level privacy constraints. 
Comparing Figure \ref{fig: privacy exp dpmac} and \ref{fig: episodic privacy exp dpmac}, 
the performance of 
DPMAC  under the episode-level privacy constraint of $\epsilon=1.00$ nearly remains the same as before. 
Under episode-level privacy constraint of $\epsilon=0.10$,
DPMAC still has a similar performance on CN and CCN tasks though converges slightly slower. But on the PP task, DPMAC also has a slight performance drop. Under 
episode-level privacy constraint of $\epsilon=0.01$, though DPMAC still has a similar performance on CN and CCN tasks, it nearly fails to learn on the PP task.
We leave the investigation of improving the performance of DPMAC under episode-level privacy constraint of $\epsilon=0.01$ as our future work.
However, we note again that the  privacy level of $\epsilon=0.01$ is a rather stringent privacy constraint requiring injecting privacy noise with a particularly large variance.
Overall, the above results show that DPMAC  can still learn to adjust the message distributions to alleviate the negative impacts of privacy noises even under episode-level privacy constraints. 

\begin{figure*}[t]
    
    \centering%
    \begin{subfigure}{1.0\textwidth}
    \includegraphics[width=1\textwidth]{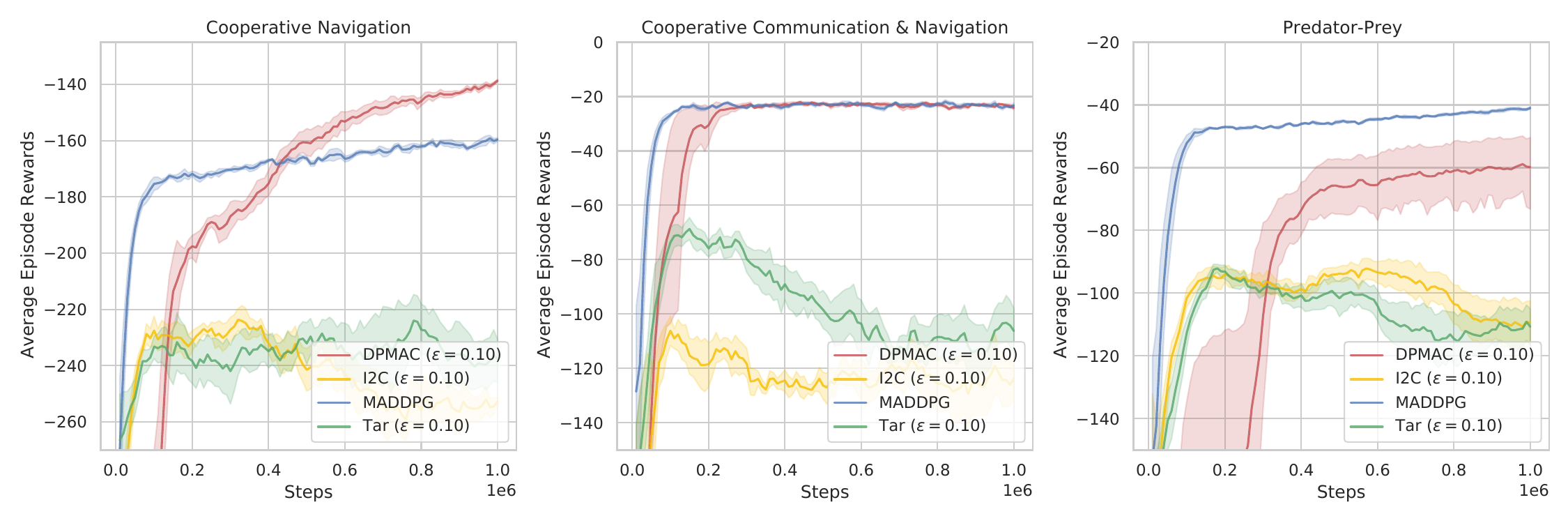}
    \caption{{\revise Performance of different algorithms under the episode-level privacy budget $\epsilon=0.10$. MADDPG (non-private) is also displayed for comparison.}}
    \label{fig: episodic privact exp 0.1}
    \end{subfigure}
    
     \centering
    \begin{subfigure}{1.0\textwidth}
        
    \includegraphics[width=1\textwidth]{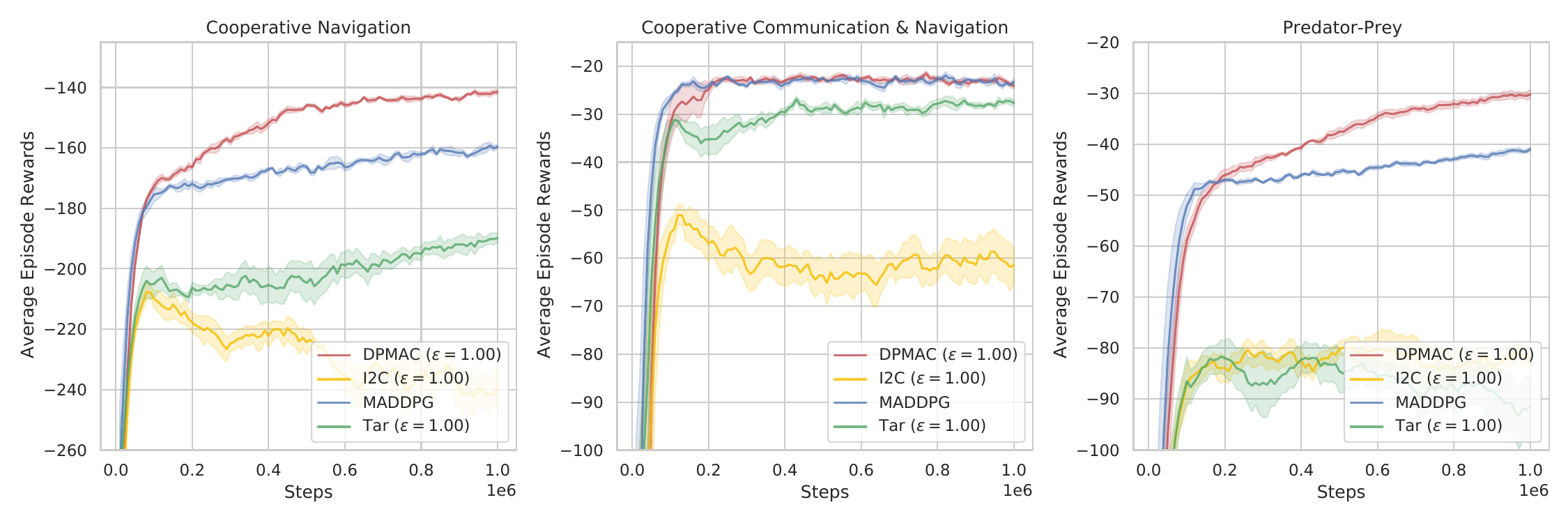}
    \caption{{\revise Performance of different algorithms under the episode-level privacy budget $\epsilon=1.0$. MADDPG (non-private) is also displayed for comparison.}}
    \label{fig: episodic privacy exp 1.0}
    \end{subfigure}
    
     \centering%
    \begin{subfigure}{1.0\textwidth}
        
    \includegraphics[width=1\textwidth]{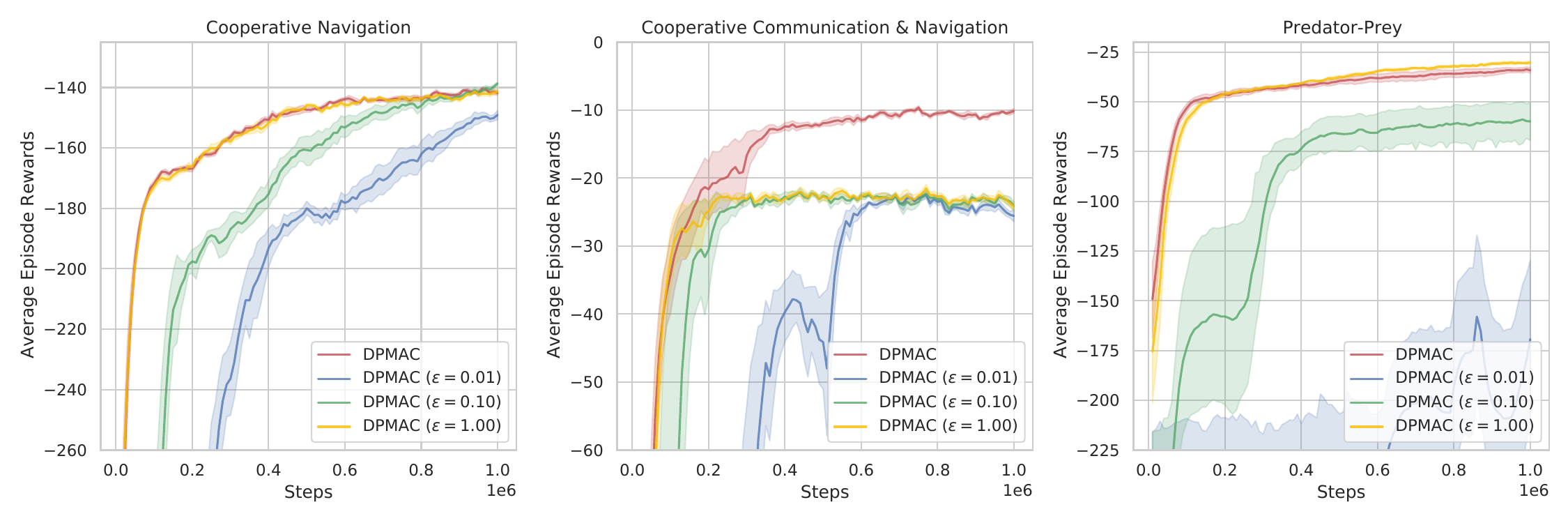}
    \caption{{\revise Performance of different episode-level privacy budgets ($\epsilon=0.01,0.10,1.00$) for DPMAC.}}
    \label{fig: episodic privacy exp dpmac}
    \end{subfigure}
    
    \caption{{\revise Performance of DPMAC and baseline algorithms. The curves are averaged over $5$ seeds. Shaded areas denote $1$ standard deviation.}}\label{fig: episodic DP of DPMAC}
\end{figure*}

}

\end{document}